\documentclass{article}

\usepackage{arxiv}

\usepackage[utf8]{inputenc} %
\usepackage[T1]{fontenc}    %
\usepackage{hyperref}       %
\usepackage{url}            %
\usepackage{booktabs}       %
\usepackage{amsfonts}       %
\usepackage{nicefrac}       %
\usepackage{microtype}      %
\usepackage{lipsum}
\usepackage{graphicx}
\usepackage{enumitem}
\usepackage{amssymb}
\usepackage{mathtools}
\usepackage{amsthm}
\newtheorem{theorem}{Theorem}[section]
\usepackage{multirow}
\usepackage[table,xcdraw]{xcolor}
\usepackage[most]{tcolorbox}
\usepackage{etoolbox}
\usepackage{subcaption}
\AtBeginEnvironment{tcolorbox}{\footnotesize}

\newcommand{\methodName}{\textsc{ConfLVLM}}
\newcommand{\modelname}[1]{\texttt{#1}}

\usepackage{amsmath,amsfonts,bm}

\def\prob{\mathbb{P}}

\def\eqref#1{equation~\ref{#1}}

\def\ceil#1{\lceil #1 \rceil}

\def\1{\bm{1}}

\def\rv{{\textnormal{v}}}

\def\rx{{\textnormal{x}}}
\def\ry{{\textnormal{y}}}

\def\vtheta{{\bm{\theta}}}

\def\mC{{\bm{C}}}

\def\mE{{\bm{E}}}

\def\mI{{\bm{I}}}

\def\mV{{\bm{V}}}

\def\mX{{\bm{X}}}
\def\mY{{\bm{Y}}}

\DeclareMathAlphabet{\mathsfit}{\encodingdefault}{\sfdefault}{m}{sl}
\SetMathAlphabet{\mathsfit}{bold}{\encodingdefault}{\sfdefault}{bx}{n}

\def\gI{{\mathcal{I}}}

\def\gL{{\mathcal{L}}}

\def\gX{{\mathcal{X}}}
\def\gY{{\mathcal{Y}}}

\def\sR{{\mathbb{R}}}

\title{Towards Statistical Factuality Guarantee for Large Vision-Language Models}

\author{
 Zhuohang Li$^1$,\; Chao Yan$^2$,\; Nicholas J. Jackson$^1$,\; Wendi Cui$^3$,\; Bo Li$^4$,\; Jiaxin Zhang$^{3,5}$,\; Bradley A. Malin$^{1,2}$ \\
  $^1$Vanderbilt University,\; $^2$Vanderbilt University Medical Center, \\
  $^3$Intuit,\; $^4$University of Illinois Urbana-Champaign,\; $^5$Intuit AI Research\\
}

\begin{document}
\maketitle

\begin{abstract}
Advancements in Large Vision-Language Models (LVLMs) have demonstrated promising performance in a variety of vision-language tasks involving image-conditioned free-form text generation. However, growing concerns about hallucinations in LVLMs, where the generated text is inconsistent with the visual context, are becoming a major impediment to deploying these models in applications that demand guaranteed reliability. In this paper, we introduce a framework to address this challenge, \methodName, which is grounded on conformal prediction to achieve finite-sample distribution-free statistical guarantees on the factuality of LVLM output. This framework treats an LVLM as a hypothesis generator, where each generated text detail (or claim) is considered an individual hypothesis. It then applies a statistical hypothesis testing procedure to verify each claim using efficient heuristic uncertainty measures to filter out unreliable claims before returning any responses to users. We conduct extensive experiments covering three representative application domains, including general scene understanding, medical radiology report generation, and document understanding. Remarkably, \methodName~reduces the error rate of claims generated by \modelname{LLaVa-1.5} for scene descriptions from 87.8\% to 10.0\% by filtering out erroneous claims with a 95.3\% true positive rate. Our results further demonstrate that \methodName~is highly flexible, and can be applied to any black-box LVLMs paired with any uncertainty measure for any image-conditioned free-form text generation task while providing a rigorous guarantee on controlling the risk of hallucination.
\end{abstract}

\section{Introduction}
\label{sec:intro}
Large Vision-Language Models (LVLMs) which  combine Large Language Models (LLMs) with computer vision modules, have demonstrated 
remarkable multi-modal abilities~\cite{liu2024visual, abdin2024phi,metaLlama32,openaigpt4omini}. LVLMs are designed to receive both free text and visual content as inputs (e.g., images or videos) and generate text responses to user queries about the visual input, thus enabling a flexible conversational interface for numerous visual perception and multi-modal comprehension tasks. This capability has triggered successful developments across 
various application domains, including cross-modality agents specializing in graphical user interface understanding and planning~\cite{hong2024cogagent}, visual-language foundation models for analyzing pathology slides~\cite{lu2024visual}, and autonomous driving assistants offering real-time reasoning and decision-making~\cite{wen2023road}.
Despite the excitement and new opportunities LVLMs offer, 
they can
produce \textit{hallucinations}~\cite{bai2024hallucination}—text outputs that contain erroneous, or simply fabricated, assertions that deviate from the visual content.  This can significantly limit their adoption in
safety-critical fields, such as healthcare and autonomous driving.

\begin{figure*}
    \centering
    \includegraphics[width=\linewidth]{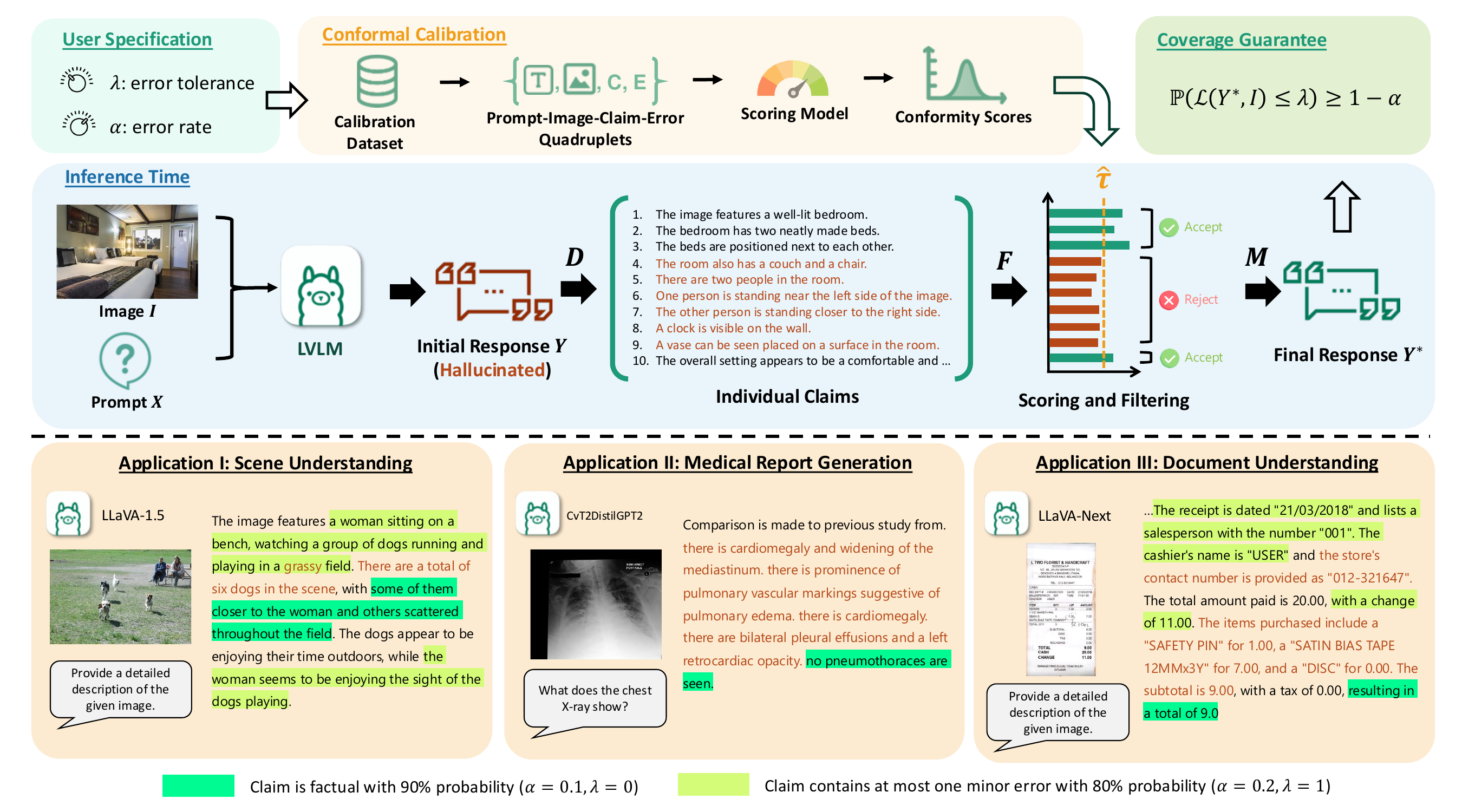}
    \caption{\textbf{Overview}: given user-specified error tolerance $\lambda$, error rate $\alpha$, and a calibration dataset, \methodName~returns a more reliable response for any new image and prompt at inference time through \textit{sampling}, \textit{decomposing} $D$, \textit{filtering} $F$, and \textit{merging} $M$, to ensure that the risk of the final response $Y^*$ is controlled with high probability. Illustrative examples, one for each application domain, are provided for outcome demonstration, where claims are highlighted to indicate \methodName's confidence using specific conformity score and error tolerance level. Unhighlighted claims correspond to low confidence in factuality check.}
    \label{fig:overview}
\end{figure*}

\paragraph{Related Work.}
To date,
the research on LVLM hallucinations has generally
focused on two
threads.
One line of work aims to build benchmarks~\cite{li2023evaluating,yin2023survey,lovenia2023negative,sun2023aligning,kaul2024throne,guan2024hallusionbench,jing2023faithscore} and metrics~\cite{rohrbach2018object,li2023evaluating,jing2023faithscore} for assessing and analyzing hallucinations in popular LVLMs.  Most studies in this area
concentrate on object hallucination in discriminative question-answering tasks~\cite{li2023evaluating,yin2023survey,lovenia2023negative},
with a lesser focus
on other types of hallucination in free-text generation tasks~\cite{jing2023faithscore,kaul2024throne,guan2024hallusionbench}.
The other line of work aims to develop strategies to mitigate LVLM hallucinations.
Notable
strategies in this line of investigation include refining the training and instruction tuning phases by optimizing the loss function~\cite{jiang2024hallucination} or alignment objective~\cite{zhao2023beyond,gunjal2024detecting,sun2023aligning}, as well as enhancing the inference phase by designing new decoding algorithms~\cite{leng2024mitigating,favero2024multi,deng2024seeing}.
However, these solutions are often resource-heavy and lack flexibility,
typically requiring model retraining or white-box access during decoding. By contrast, \emph{post hoc} correction methods offer a more flexible approach,  which improves responses from any black-box LVLMs through the assistance of external language or vision modules~\cite{yin2023woodpecker}, dedicated revisor model~\cite{zhou2023analyzing}, or self-revision techniques~\cite{lee2023volcano}. However, the current collection of methods relies solely on heuristics and lacks rigorous statistical guarantees of factuality for the revised output, a necessity in mission-critical application domains.

\paragraph{Our Contribution.}
To address these challenges, in this study, we introduce \methodName, a framework with statistical guarantees on output factuality (alignment of text response with visual context) that seamlessly integrates with LVLMs of any architecture, complexity, and purpose.
\methodName~treats LVLM-generated free text as a series of individual claims, each corresponding to a testable hypothesis. Each claim is then evaluated for its factuality using discriminative mechanisms built from the same (or auxiliary) LVLM to filter out unsupported claims and retain those that meet factual standards. To provide statistical guarantees on the factuality of retained claims, \methodName~employs a conformal prediction framework that allows control for flexible error rates, as well as error tolerance levels, defined by users
to suit specific application needs. By leveraging the generally more stable discriminative capabilities of LVLMs through calculating and ranking predefined conformity scores, \methodName~enables quantitative factual assessments. We demonstrate the effectiveness of \methodName~across three representative application domains, i.e., general scene understanding, medical radiology report generation, and document understanding,
to validate its effectiveness in ensuring a desired level of response factuality for various state-of-the-art LVLMs and explore multiple potentially useful conformity scores. Our results, which cover over $81,000$ claims generated from eight popular LVLMs, establish \methodName~as a general-purpose framework that operates with any LVLM in an assumption- and finetuning-free manner, thus promoting trustworthiness in LVLM applications broadly. Moreover, this work opens up a novel research space, where each component of \methodName~invites further innovation to improve the rigorous guarantee of hallucination mitigation for multi-modal models.

\section{Preliminaries}
\label{sec:prelim}

\subsection{Large Vision-Language Models}

Large Vision-Language Models (LVLM) are generative models that are typically composed of a visual model $h(\cdot)$, a language model parameterized by $\vtheta$, and a fusion model $g(\cdot)$.
The most popular implementations of LVLMs, such as Llava~\cite{liu2024visual}, combine a pre-trained visual encoder (e.g., CLIP~\cite{radford2021learning}) and a pre-trained Large Language Model (LLM) (e.g., Vicuna~\cite{chiang2023vicuna}) by training a projection network as the fusion model to convert extracted visual features into the LLM's embedding space in a process known as visual instruction tuning.
During inference, an LVLM takes an input image $\mI$ and a text prompt $\mX=[\rx_1, ..., \rx_l]$, and outputs a text response $\mY=[\ry_1, ..., \ry_m]$, where $\rx_i$ and $\ry_j$ are individual tokens. This is achieved by first converting the image into a sequence of visual tokens using the visual model $\mV=[\rv_1, ..., \rv_k]=g\circ h(\mI)$ and then sampling the response from the conditional distribution in an autoregressive manner:
$p_\vtheta(\mY|\mX,\mV) = \prod_{j=1}^m p_\vtheta(\ry_j|\mX,\mV,\mY_{<j})$.

\paragraph{Hallucination of LVLMs.}
The problem of hallucination originates from the space of language models, where the generated text response is either non-factual (conflicts with verifiable facts) or unfaithful (does not follow the user's instructions). In the context of LVLMs, hallucination refers to the phenomenon where the generated text response deviates from the provided visual content. Common types of LVLM hallucinations include \textit{object} hallucination (e.g., falsely identifying non-existent objects), \textit{attribute} hallucination (e.g., wrong color, shape, or material),
and \textit{relation} hallucination (e.g., human-object interaction, relative position)~\cite{bai2024hallucination}.

\subsection{Split Conformal Prediction}

Split conformal prediction (SCP)~\cite{vovk2005algorithmic,shafer2008tutorial} is a distribution-free method for quantifying the uncertainty of black-box prediction algorithms by constructing prediction sets with finite-sample coverage properties. 

\paragraph{Coverage Guarantee.}
For a black-box prediction function $f: \gX \rightarrow \gY$, let $\{(X_i, Y_i)\}_{i=1}^{n+1}$ be an exchangeable set of feature and label pairs sampled from the joint distribution on $\gX\times\gY$. The goal of split conformal prediction is to use the calibration data $\{(X_i, Y_i)\}_{i=1}^{n}$ and $f$ to construct a prediction set $\hat{C}: \gX \rightarrow 2^\gY$ for the new data point such that it achieves valid \textit{coverage}, i.e., containing the true label with high probability
$\prob\big(Y_{n+1}\in \hat{C}(X_{n+1})\big) \geq 1-\alpha$
for any user-specified error rate $\alpha \in (0, 1)$.

\paragraph{Conformal Calibration.}
Suppose there is a \textit{conformity score} function $S(X, Y)\in \sR$ that measures how well a given sample \textit{conforms} to the observed data.
The split conformal procedure uses the calibration data set $\{(X_i, Y_i)\}_{i=1}^{n}$ to derive \textit{conformity} scores $\{S(X_i, Y_i) \}_{i=1}^n$, where a larger value indicates the model is more confident about the prediction being true. To calibrate the prediction set to the desired level of coverage, we then compute a threshold $\hat{\tau}$
that is approximately the $1-\alpha$ quantile of the conformity scores.
At the time of inference, given a new data point $X_{n+1}$, we construct the prediction set as $\hat{C}(X_{n+1}) = \{y\in\gY: S(X, y) \geq \hat{\tau} \}$. If the data are exchangeable, then this prediction set will satisfy the desired coverage property.

\section{Ensuring Factuality for LVLMs}
\subsection{Problem Formulation}
Given a pair of image and text prompt $(I_{n+1}, X_{n+1})$
and a set of $n$ calibration data points, our goal is to generate a reliable response $Y^*_{n+1}$ using the LVLM, such that it contains a low error
with high probability; i.e.,
\begin{equation}\label{eq:obj}
    \prob(\gL(Y^*_{n+1}, I_{n+1}) \leq \lambda) \geq 1- \alpha,
\end{equation}
where $\gL:\gY \times \gI \rightarrow \sR^+_0$ is a monotonic \textit{loss} function that measures the level of misalignment between the statement and the image (e.g., the occurrence of object hallucination, or inaccuracy in item attribute or quantity)
and $\lambda$ is a user-specified tolerance. A larger loss indicates a greater amount of error, while a loss of zero indicates that the statement made in the response is \textit{factual} with respect to the provided image. %
We define $\gL(\varnothing, \cdot) = 0$, which indicates that we do not penalize the model for abstaining from responding when it is uncertain, as here we only focus on the assurance of factuality while neglecting other aspects of reliability such as omission of information.

\subsection{Error Control in LVLMs}

Due to the open-set and free-form nature of the natural language output, directly attempting to construct prediction sets is not attainable for generative models like LVLMs.
Instead, we adapt the recently proposed \textit{conformal factuality}~\cite{mohrilanguage,cherian2024large} framework to the multi-modal setting as the central tool for achieving statistical guarantees on the factuality of LVLM outputs.
The key idea is to exploit the connection between linguistic entailments and uncertainty sets to back off from the original statement (uncensored response from LVLM)
by gradually removing unreliable claims with high uncertainty until the desired level of correctness is achieved.

\paragraph{Error Control Procedure.}
We start by defining a scoring function $r(C^j_{n+1}, I_{n+1})\in \sR$ that captures the system's
confidence about the claim concerning the provided visual context, where a larger score indicates that the claim is more aligned with the provided image and, thus, is more likely to be true.
Given $I_{n+1}$ and $X_{n+1}$, we execute the following steps to generate a more reliable response $Y^*_{n+1}$:

\begin{enumerate}[label={(\arabic*)}]
    \item \textit{Initial Hypothesis Generation}: Sample an initial response $Y_{n+1}$ from the LVLM $p_\vtheta(Y|X_{n+1},V_{n+1})$, where $V_{n+1}=g\circ h(I_{n+1})$.
    \item \textit{Decomposition}: Apply a \textit{decomposition}
    operator $D$ to breakdown the initial response into a set of individual claims: $\mC_{n+1} = D(Y_{n+1}) = \{C^j_{n+1}\}^{s_i}_{j=1}$.
    \item \textit{Individual Hypothesis Testing}: Define the \textit{filtering} operator as $F(\mC_{n+1}; \tau) \coloneq \{C^j_{n+1}: r(C^j_{n+1}, I_{n+1}) > \tau\}$. Generate a filtered set of claims $F(\mC_{n+1}; \tau) \subseteq \mC_{n+1}$.
    This step can be thought of as testing each individual hypothesis $C^j_{n+1}$ and accepting the hypothesis only if the test statistic $r(C^j_{n+1}, I_{n+1})$ is greater than the chosen threshold $\tau$.
    \item \textit{Combination}: Apply a \textit{merge}
    operator $M$ to combine the filtered claims into the final response $Y^*_{n+1} = M\big(F(\mC_{n+1}; \tau)\big)$.
\end{enumerate}

\paragraph{Calibration.}
Next, to calibrate the filtering operator $F$, we set the conformity score $S$ for each set of claims to be the minimum threshold that ensures the loss of the filtered set of claims is controlled to be within tolerance:
\begin{equation*}
    S(\mC_{n+1}, I_{n+1})=\inf\{ \tau: \gL \big(F(\mC_{n+1}; \tau), I_{n+1} \big) \leq \lambda\}.
\end{equation*}
Finally, we implement the calibrated filtering operator as $\hat{F}(\mC_{n+1}) \coloneq F(\mC_{n+1}; \hat{\tau}) = \{C^j_{n+1}: r(I, C^j_{n+1}) > \hat{\tau}\}$, where $\hat{\tau}$ is set to be the $\frac{\ceil{(n+1)(1-\alpha)}}{n}$-th quantile of the conformity scores $\{S(\mC_{i}, I_{i}) \}_{i=1}^n$ estimated on the calibration dataset.

The following %
theorem indicates
that if the data are exchangeable, the response produced using the calibrated filtering operator 
will satisfy the error control objective in Ineq.~\ref{eq:obj}.

\begin{theorem}[SCP Coverage Guarantee~\cite{shafer2008tutorial,mohrilanguage}]
Define the error scores $\mE_{i} \coloneq \{\gL (C_i^j, I_i): C_i^j \in \mC_i\}$.
Let $\{(X_i, I_i,\mC_i, \mE_i)\}_{i=1}^{n+1}$ be exchangeable, then the following lower bound holds for any $\alpha \in (\frac{1}{n+1}, 1)$:
\begin{equation*}
     \prob\Big(\gL\big(\hat{F}(\mC_{n+1}), I_{n+1}\big) \leq \lambda\Big) \geq 1-\alpha.
\end{equation*}

If the loss function is monotonic, meaning that $\gL(\hat{F}_1(\mC_i, I_i)) \leq \gL(\hat{F}_2(\mC_i, I_i))$ for any $\hat{F}_1(\mC_i, I_i) \subseteq \hat{F}_2(\mC_i, I_i)$, then the following upper bound also holds:
\begin{equation*}
     1 - \alpha + \frac{1}{n+1} \geq \prob\Big(\gL\big(\hat{F}(\mC_{n+1}), I_{n+1}\big) \leq \lambda\Big).
\end{equation*}

\end{theorem}

\begin{proof}
Without loss of generality, let us assume that the conformity scores are sorted as $s_1 < s_2  < ... s_n$, where $s_i = S(\mC_{i}, I_{i})$.
Notice that under the definition of $S$, the event $\{s_{n+1} \leq \hat{\tau}\}$ implies $\{ \gL \big(\hat{F}(\mC_{n+1}), I_{n+1} \big) \leq \lambda \}$.
By exchangeability, $\prob(s_{n+1} \leq s_{\ceil{(n+1)(1-\alpha)}}) = \frac{\ceil{(n+1)(1-\alpha)}}{n+1} \geq 1 - \alpha$, which implies the result.
To prove the upper bound, notice that the two events $\{s_{n+1} \leq \hat{\tau}\}$ and $\{ \gL \big(\hat{F}(\mC_{n+1}), I_{n+1} \big) \leq \lambda \}$ are now equivalent if the loss function is monotone. The result can then be obtained through $\prob(s_{n+1} \leq s_{\ceil{(n+1)(1-\alpha)}}) = \frac{\ceil{(n+1)(1-\alpha)}}{n+1} \leq \frac{(1-\alpha)(n+1) +1}{n+1} = 1 - \alpha + \frac{1}{n+1}$.

\end{proof}

\subsection{Deriving Conformity Scores}

In practice, the decomposition operator can be implemented by prompting the language model part of the LVLM, which does not rely on any external resources.
To derive the conformity scores, we will need to find a suitable scoring function $r(C_i, I_i)$.
Built on conformal prediction, our framework should maintain valid coverage with any arbitrary heuristic scores. However, in practice, a score that better captures the relevance between a claim $C_i$ and the given image $I_i$ can enable a better tradeoff (i.e., allowing the same coverage guarantee while filtering out less content).
We primarily consider the following two types of scores.

\paragraph{Internal Scores.}
We consider the following scores to capture the internal confidence of an LVLM regarding a statement.
(1) \textit{Log Probability of Text Tokens}: we compute the log probability of text tokens from the claim given only the text prompt as $r(C_i, I_i) = \log p_\vtheta(C_i|X_i)$, which does not make use of the visual context and thus serves as a language prior baseline.
(2) \textit{Log Probability of Text Tokens Conditioned on Image}: we compute the log probability of text tokens from the claim conditioned on both the visual and the text prompt as $r(C_i, I_i) = \log p_\vtheta(C_i|X_i, V_i)$, which is the visual instruction tuning objective~\cite{liu2024visual} of most LVLMs.
(3) \textit{Log Probability Ratio}: finally, we consider the ratio between the two probabilities, i.e.,  $r(C_i, I_i) = \log \frac{p_\vtheta(C_i|X_i, V_i)}{p_\vtheta(C_i|X_i)}$. This is motivated by the observation that most hallucinations in LVLMs occur because their language prior tend to dominate visual perception during decoding~\cite{favero2024multi,leng2024mitigating,liu2024paying}, and the probability ratio can be an informative measure of the true influence of the visual prompt regardless of the language prior.

\paragraph{External Scores.}
In addition to internal scores, we also consider capturing the confidence in a statement regarding an image using a (lightweight) external model. \textit{Energy-based models}~\cite{lecun2006tutorial} $E: \gY \times \gI \rightarrow \sR$ are particularly suited for this task as they are trained to map image-text pairs to a scalar energy score so that $e^{-E(Y, I)} \propto \prob(Y, I)$. As such, we can simply set the scoring function to return the negative energy score, 
$r(C_i, I_i)=-E(C_i, I_i)$.

\section{Case Study I: General Scene Understanding}
\subsection{Setup}

We first evaluate 
the LVLM factuality framework for general scene understanding tasks.
To do so, we use 
$500$ randomly selected images from the MSCOCO~\cite{lin2014microsoft} validation set with more than three objects (same as the POPE dataset~\cite{li2023evaluating}).

\paragraph{LVLMs.}
We use four start-of-the-art LVLMs for the evaluation of the scene understanding task, including three open-sourced models \modelname{LLaVA-1.5}~\cite{liu2024visual}, \modelname{Phi-3.5-vision-instruct}~\cite{abdin2024phi}, \modelname{Llama-3.2-11B-Vision}~\cite{metaLlama32}, and one close-sourced model \modelname{GPT-4o-mini}~\cite{openaigpt4omini}.
We prompt each LVLM to generate a detailed description for each image and decompose the original response into independent claims.
Our prompts are provided in the Appendix.

\paragraph{Error Annotation and Loss Function.} 
We consider the following five types of errors for scene understanding:
(1) \textit{Object identification}: The claim involves hallucinated or wrongly identified objects;
(2) \textit{Attribute (in)accuracy}: The claim involves incorrect attributes (e.g., color, size, shape);
(3) \textit{Spatial relations}: The claim involves incorrect spatial relationships between objects;
(4) \textit{Interaction/Action (in)accuracy}: The claim involves incorrect or hallucinated action or interaction;
and (5) \textit{Quantitative information}: The claim involves incorrect numeric details (e.g., the wrong object count).
We prompt \modelname{GPT-4o} to label each claim as either correct or belonging to one or more error categories.
In practice, the loss function can be tailored according to %
specific use cases. For 
our experiments, we consider a cumulative loss function that assigns a loss score to each response based on the total error its claims contain. Specifically, all claims start with a loss score equal to $0$. For each ``Object identification'' error contained in a claim, the loss is increased by $3$; for every other type of error contained, the loss is increased by $1$. A correct claim thus receives a loss of $0$.
This choice of loss structure reflects the common consensus that hallucinating non-existing objects is typically more harmful compared to other types of hallucinations.

\paragraph{Scoring Function.}
We use two configurations of pretrained CLIP~\cite{radford2021learning} models, \modelname{CLIP-ViT-Base} with $32px$ patch size and \modelname{CLIP-ViT-Large} with $14px$ patch size, to derive the normalized image and text embeddings and compute the dot product as external confidence scores.

\begin{figure}
    \centering
\includegraphics[width=0.56\linewidth]{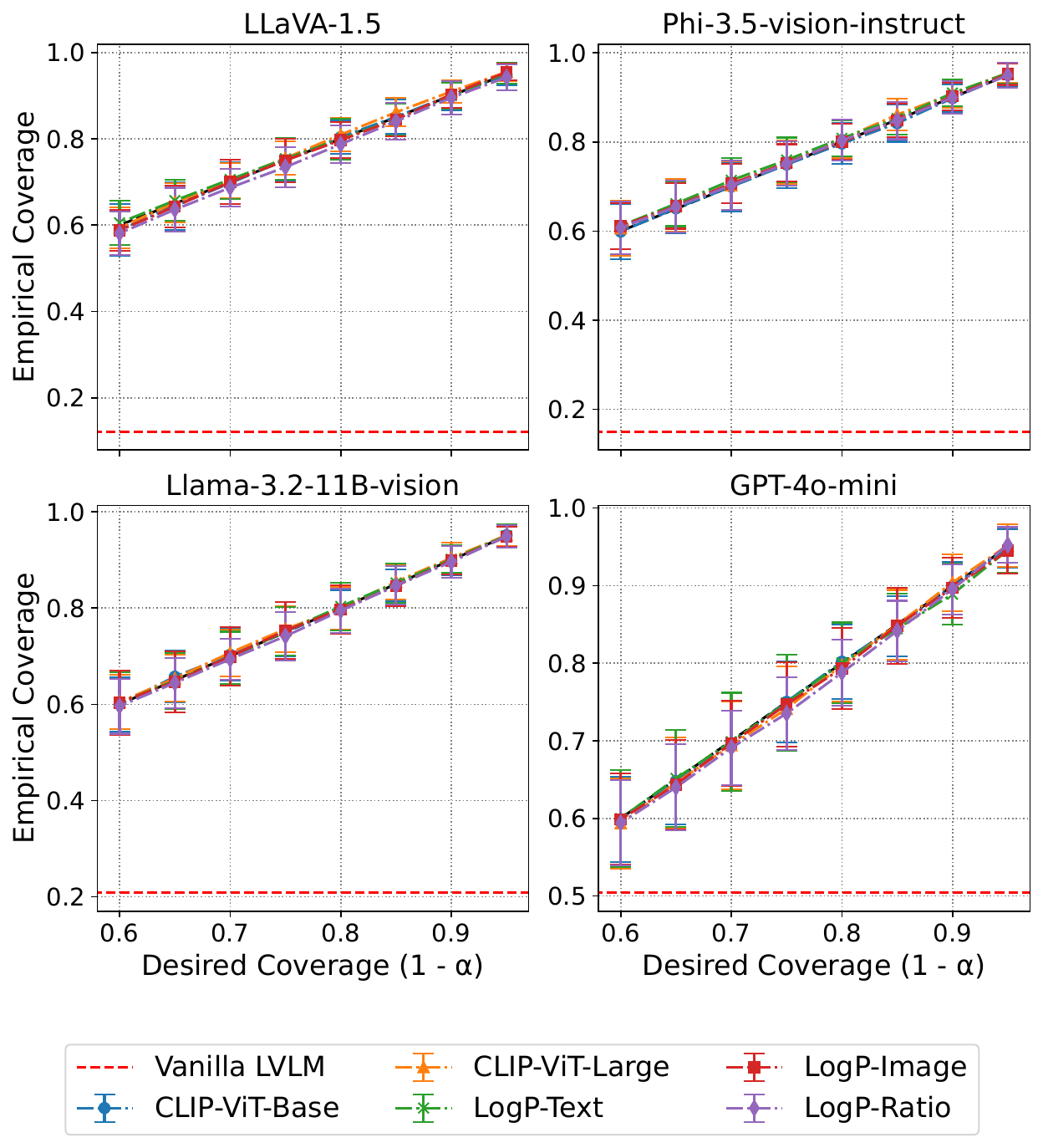}
    \caption{Alignment between empirical and desired (theoretical) coverage in \textit{scene understanding} (with $\lambda=0$).
    Vanilla LVLM (red dashed line) refers to the base setting where the LVLM-generated responses are returned to users without using \methodName.
    }
    \label{fig:pope_coverage}
\end{figure}

\subsection{Results}

\paragraph{\methodName~Achieves Any Desired Level of Coverage.}
We
examine the validity of \methodName~by measuring the empirical coverage (ratio of responses that satisfy $\gL\big(\hat{F}(\mC_{n+1}), I_{n+1})\big) \leq \lambda$ over total number of responses) under different levels of desired (theoretical) coverage determined by $1-\alpha$.
We set the error tolerance $\lambda$ to $0$ (most restrictive) and report the results over $50$ random splits of calibration ($400$ data points for establishing conformal prediction) and test data ($100$ data points for computing the empirical coverage). %
The results shown in Fig.~\ref{fig:pope_coverage} confirm that \methodName~can
achieve the desired level of coverage with all types of scoring functions and for all LVLMs considered. 
By contrast, Vanilla LVLM (i.e., responses without any filtration) leads to significantly low coverage that signals a failure in the model's reliability.

\begin{figure}
    \centering
\includegraphics[width=0.56\linewidth]{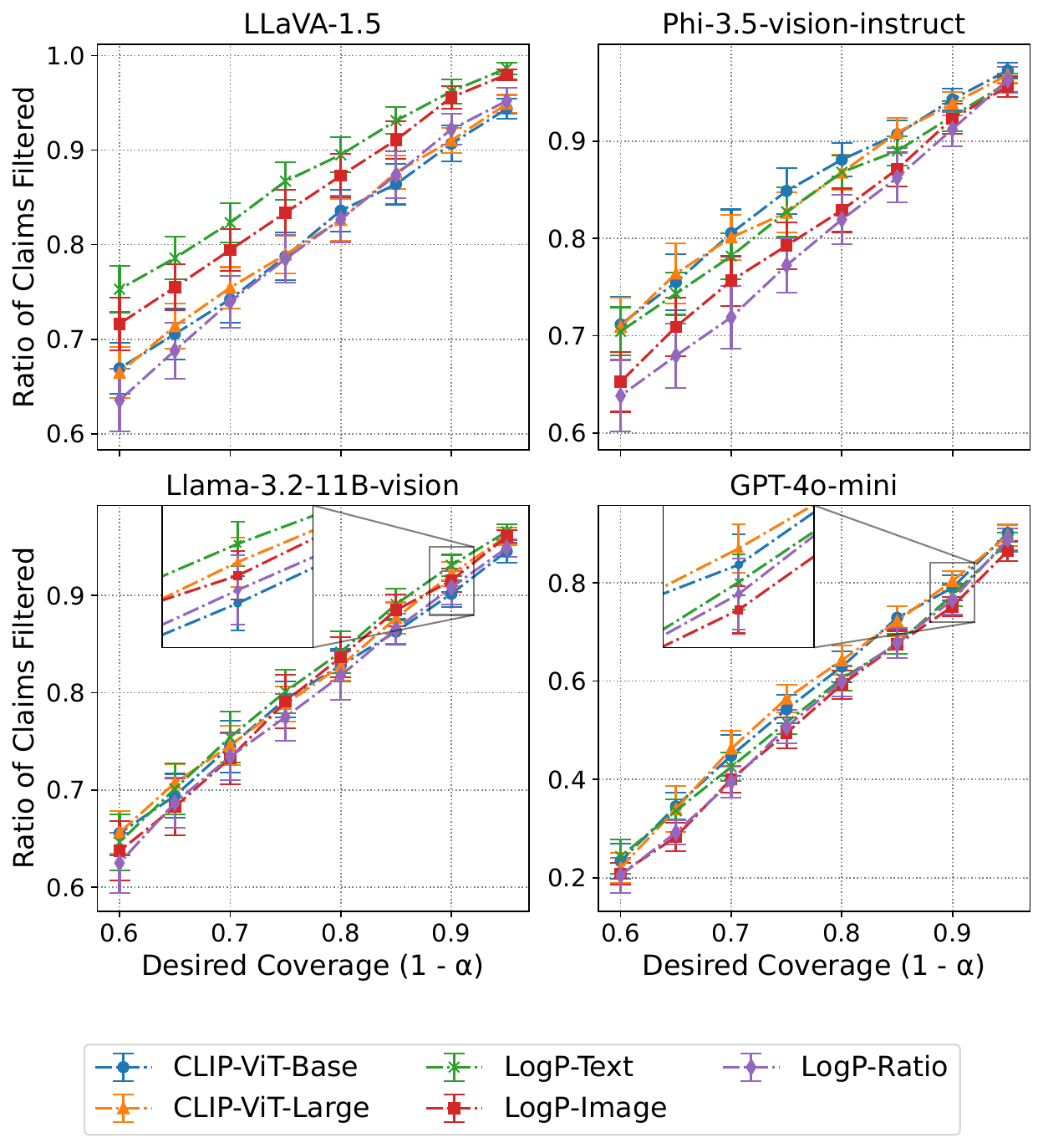}
    \caption{Average ratio of claims filtered with varying coverage using different scoring functions in \textit{scene understanding} (with $\lambda=0$). Standard errors are marked.}
    \label{fig:pope_cont}
\end{figure}

\paragraph{\methodName~Achieves Higher Filtering Efficiency Than Baselines.}
Approaching the desired coverage level of $1-\alpha$ involves flagging and filtering out low-confidence claims and, in some cases, abstaining from providing a response.
Next, we analyze the ratio of claims being filtered out by \methodName~(Fig.~\ref{fig:pope_cont}) and the associated rate of abstention  (Fig.~\ref{fig:pope_abstent}).
We observe a general trend where the ratio of filtered claims and the abstention rate increase as the desired coverage level rises.
This is expected as the unfiltered responses from LVLMs contain a large number of non-factual claims (e.g., $87.8\%$ responses from \modelname{LLaVA-1.5} are erroneous), and thus ensuring a lower error rate requires the framework to be more conservative and filter more content. When comparing across different scoring functions, it can be seen that, in the case of \modelname{LLaVA-1.5}, the external confidence scores based on \modelname{CLIP} models 
achieve a lower ratio of filtered claims and also a lower abstention rate than
internal scores. This is notable
because \modelname{LLaVA-1.5} uses \modelname{CLIP} as the visual encoder, which implies that there may be certain deficiencies in the visual instruction tuning process.
However, such a trend is weaker or non-existent
with other LVLMs. This may be because claims generated by other LVLMs are typically longer and contain more details, which is 
more difficult to capture using \modelname{CLIP}.
When comparing across LVLMs, we see that \modelname{GPT-4o-mini} requires filtering out much fewer claims to achieve the desired coverage, which is because \modelname{GPT-4o-mini} has better empirical factuality performance compared to other models.

\begin{figure}
    \centering
\includegraphics[width=0.56\linewidth]{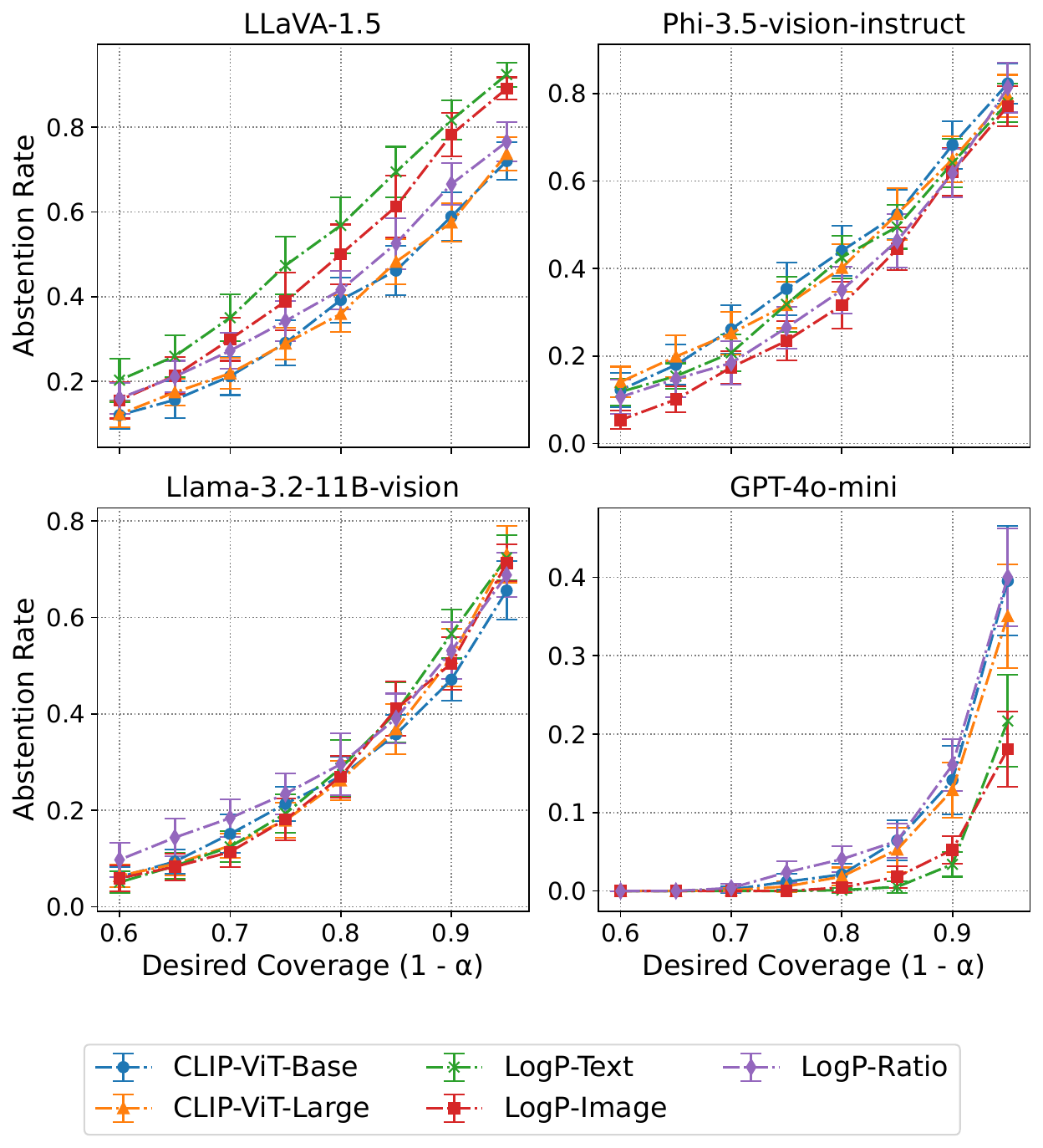}
    \caption{Abstention rate with varying coverage using different scoring functions in \textit{scene understanding} (with $\lambda=0$).}
    \label{fig:pope_abstent}
\end{figure}

Now we analyze how accurate \methodName~is in filtering out nonfactual claims.
To our knowledge, there are no existing baselines with statistical factuality guarantees. Thus, we consider a simple \textit{Random Filtering} baseline beyond Vanilla LVLM, which drops claims uniformly at random with probability $\alpha$.
As shown in Table~\ref{tab:pope_res}, for \modelname{LLaVA-1.5} with the setting of $\alpha=0.1, \lambda=0$, \methodName~achieves a high true positive rate (TPR, or Recall) of $0.953$ and a relatively high F1 score of $0.504$. Vanilla LVLM, in contrast, has a TPR of $0$ and F1 of $0$, as it does not perform any filtration. When randomly dropping $10\%$ claims, Random Filtering reaches a TPR of $0.104$ and an F1 of $0.158$, which are approximately $9\times$ and $3.2\times$ lower than those achieved by \methodName. We have similar observations for other LVLMs and these highlight the utility of \methodName~in identifying nonfactual claims. 
It is also evident that other LVLMs with \methodName~implemented demonstrate lower values of TPR and F1 compared to \modelname{LLaVA-1.5}, e.g., \modelname{GPT-4o-mini} achieves a TPR of $0.850$ and an F1 of $0.100$. This occurs because models like \modelname{GPT-4o-mini} typically generate more comprehensive descriptions of an image, leading to a higher number of claims, and, at the same time, a smaller proportion of nonfactual claims
(examples in Appendix).
In other words, the task of identifying nonfactual claims itself is much harder for \modelname{GPT-4o-mini} than for \modelname{LLaVA-1.5}. Nonetheless, \methodName~still outperforms the baseline approaches. 

\begin{table}[t]
    \centering
    \caption{Claim-level results on the \textit{scene understanding} task using \modelname{CLIP-ViT-Large} as scoring function (\methodName~with $\alpha=0.1,\lambda=0$).}
    \vspace{2mm}
    \label{tab:pope_res}
    \resizebox{0.56\linewidth}{!}{
\begin{tabular}{l|c|cc}
\toprule
LVLM                    & \textbf{Configuration}                                                                & \textbf{TPR$\uparrow$}                                    & \textbf{F1$\uparrow$}                    \\ \toprule
\modelname{LLaVA-1.5}               & Vanilla LVLM                                                                               & 0.0                                                        & 0.0                            \\
& \begin{tabular}[c]{@{}c@{}}Random Filtering\end{tabular}                                                                              & 0.104                                                       & 0.158                            \\
                        & \cellcolor[HTML]{EFEFEF}\begin{tabular}[c]{@{}c@{}}\methodName\end{tabular} & \cellcolor[HTML]{EFEFEF}\textbf{0.953} & \cellcolor[HTML]{EFEFEF}\textbf{0.504} \\ \midrule
\begin{tabular}[l]{@{}c@{}}\modelname{Phi-3.5-vision}\end{tabular} & Vanilla LVLM                                                                               & 0.0                                                       & 0.0                            \\
\modelname{-instruct} & \begin{tabular}[c]{@{}c@{}}Random Filtering\end{tabular}                                                                               & 0.102                                                        & 0.145                            \\
                        & \cellcolor[HTML]{EFEFEF}\begin{tabular}[c]{@{}c@{}}\methodName\end{tabular} & \cellcolor[HTML]{EFEFEF}\textbf{0.945}  & \cellcolor[HTML]{EFEFEF}\textbf{0.401} \\ \midrule
\modelname{Llama-3.2-11B-vision}    & Vanilla LVLM                                                                              & 0.0                                                       & 0.0                            \\
& \begin{tabular}[c]{@{}c@{}}Random Filtering\end{tabular}                                                                               & 0.099                                                       & 0.121                            \\
                        & \cellcolor[HTML]{EFEFEF}\begin{tabular}[c]{@{}c@{}}\methodName\end{tabular} & \cellcolor[HTML]{EFEFEF}\textbf{0.936} & \cellcolor[HTML]{EFEFEF}\textbf{0.269} \\ \midrule
\modelname{GPT-4o-mini}             & Vanilla LVLM                                                                               & 0.0                                                   & 0.0                            \\
& \begin{tabular}[c]{@{}c@{}}Random Filtering\end{tabular}                                                                               & 0.106                                                      & 0.070                            \\
                        & \cellcolor[HTML]{EFEFEF}\begin{tabular}[c]{@{}c@{}}\methodName\end{tabular} & \cellcolor[HTML]{EFEFEF}\textbf{0.850} & \cellcolor[HTML]{EFEFEF}\textbf{0.100} \\ \bottomrule
\end{tabular}
    }
\end{table}

\paragraph{Error Tolerance Allows Flexible Control Over Coverage-Utility Tradeoff.}
To study the impact of error tolerance, we plot the ratio of claims filtered and abstention rate using \modelname{Llama-3.2-11B-Vision} with varying $\lambda$ while keeping $\alpha=0.1$ in Fig.~\ref{fig:pope_err_a}. We observe the expected behavior that \methodName~will filter out less content and abstain less as the error tolerance increases.
In Fig.~\ref{fig:pope_err_b}, we plot the model's response distributions with different error tolerances and using \modelname{CLIP-ViT-Large} as the scoring function.
When the error tolerance is set to $\infty$, i.e., using the raw LVLM output without censoring, the model can generate detailed responses ($75\%$ responses have $\geq 10$ claims) but also have high risks of hallucination (more than $50\%$ responses have loss $\geq 3$).
As the error tolerance decreases, the loss of responses gradually reduces with the number of claims. For example, an error tolerance is set to $3$, which reduces more than $75\%$ of responses to below $2$, with the median of the number of claims being $4$.
This shows that error tolerance can serve as an additional tuning knob that allows the user to flexibly choose the desired level of error within the acceptable range of utility.

\section{Case Study II: Medical Report Generation}
\subsection{Setup}
Next, we evaluate \methodName~on the radiology report generation task. For this evaluation, we use a subset of $500$ chest X-ray images from the MIMIC-CXR~\cite{johnson2019mimic} dataset, each from a distinct patient.

\paragraph{LVLMs.} We consider the following three medical-domain LVLMs for this task:
(1) \modelname{LlaVa-Med}~\cite{li2024llava} is a biomedical LVLM instruction-tuned on several corpora in the biomedicine domain. We use the latest v1.5 \modelname{Mistral} 7B version.
(2) \modelname{CvT2DistilGPT2}~\cite{nicolson2023improving} is LVLM based on the encoder-to-decoder architecture developed for chest X-ray report generation. The originally released model weights are trained on the MIMIC-CXR dataset. To avoid data leakage, we retrain the model on a disjoint subset of MIMIC-CXR that does not contain any patients involved in our evaluation.
(3) \modelname{MAIRA-2}~\cite{bannur2024maira} is the latest radiology-specific LVLM developed by Microsoft Research. It is based on a similar architecture as \modelname{LlaVa}, featuring a \modelname{Rad-DINO} visual encoder and a language model based on \modelname{Vicuna} 7B v1.5 for grounded report generation.

\begin{figure}
    \centering
    \begin{subfigure}[b]{\linewidth}
        \centering
        \includegraphics[width=0.76\linewidth]{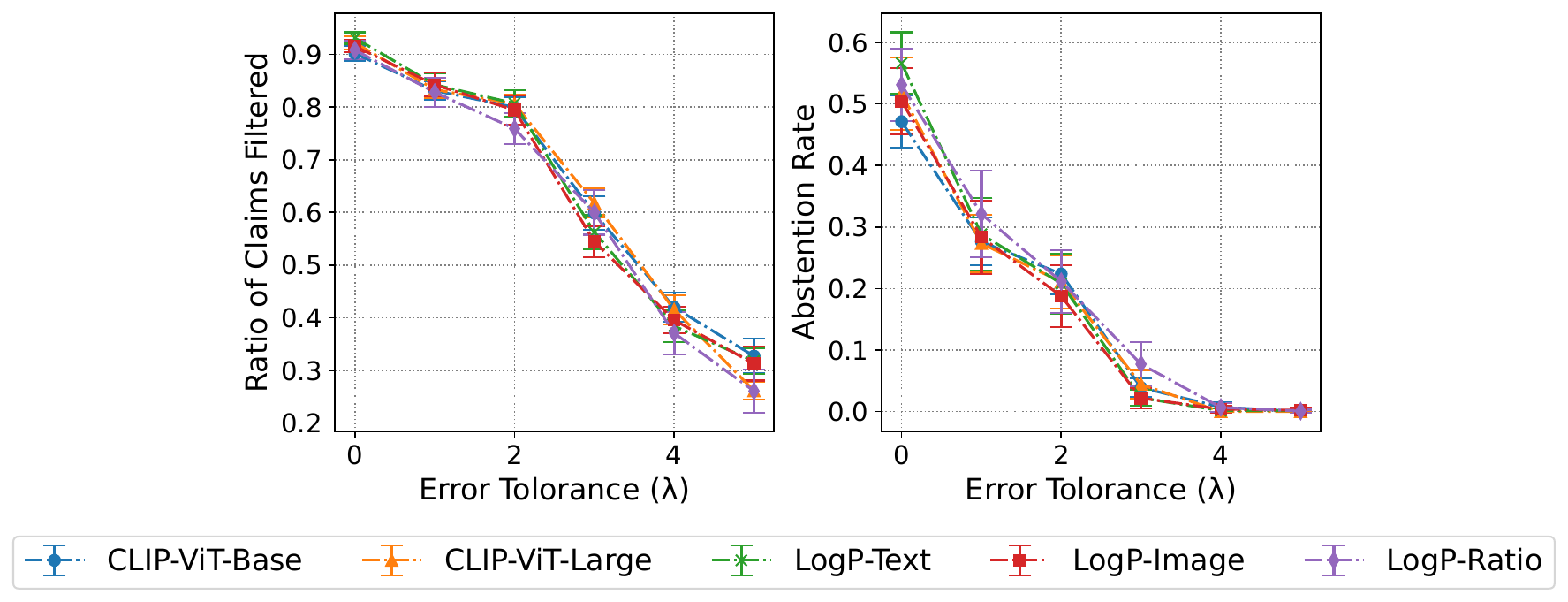}
        \caption{Ratio of claims filtered and abstention rate}
        \label{fig:pope_err_a}
    \end{subfigure}

    \begin{subfigure}[b]{\linewidth}
        \centering
        \includegraphics[width=0.54\linewidth]{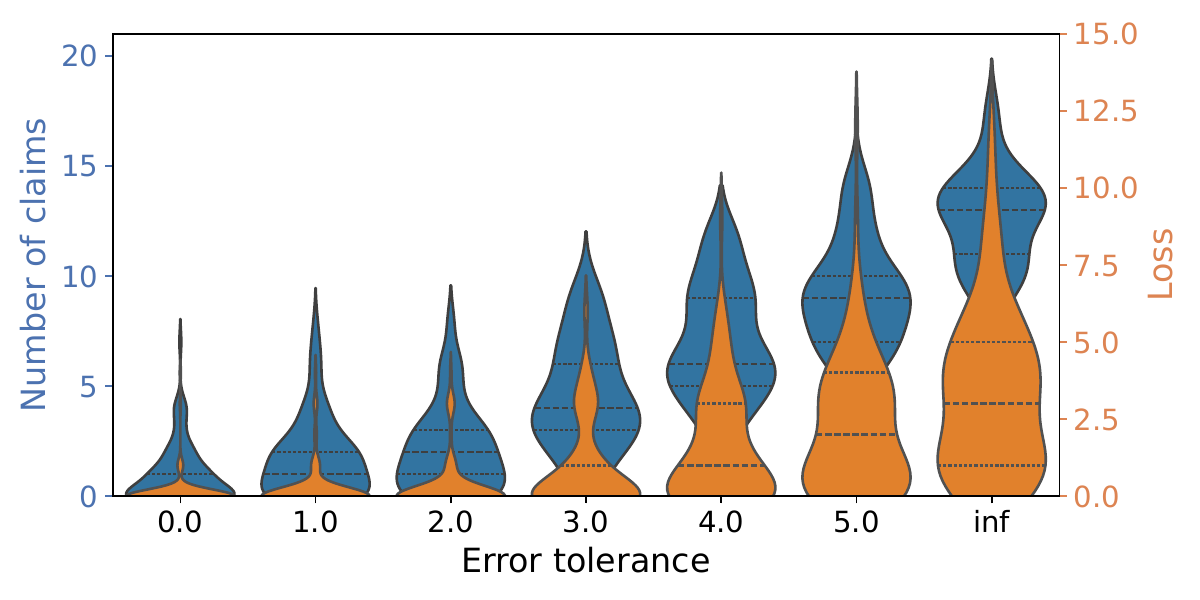}
        \caption{Response distributions using \modelname{CLIP-ViT-Large} as scoring function}
        \label{fig:pope_err_b}
    \end{subfigure}
    \caption{Comparison of \modelname{Llama-3.2-11B-Vision}'s response with different error tolerances ($\lambda$) while fixing $\alpha=0.1$ in \textit{scene understanding}.}
    \label{fig:pope_err}
\end{figure}

\paragraph{Error Annotation and Loss Function.}
We leverage \modelname{GPT-4o} for error annotation. To ensure quality, we provide \modelname{GPT-4o} with the ground truth report written by qualified physicians in addition to the chest X-ray images. This eliminates the requirement for \modelname{GPT-4o} to understand the actual medical image as it can verify the veracity of each claim by just checking if it is entailed by the ground truth report.
Specifically, the errors are categorized as:
(1) \textit{Conflicting error}: The claim directly contradicts information provided in the ground truth report;
(2) \textit{Implausible error}: The claim does not directly conflict with or align with the ground truth report, and is implausible within the given context;
and (3) \textit{Plausible error}: The claim does not directly conflict with or align with the ground truth report, but remains plausible within the given context.
Similar to scene understanding, we assign each occurrence of errors (1)-(3) a loss of $3$, $2$, and $1$, respectively, and compute the accumulated loss as the final loss for each response.

\paragraph{Scoring Function.}
We use \modelname{BiomedCLIP}~\cite{zhang2023biomedclip} to compute the similarity between claim text and image pairs as the external confidence score.

\begin{figure}
    \centering
\includegraphics[width=0.8\linewidth]{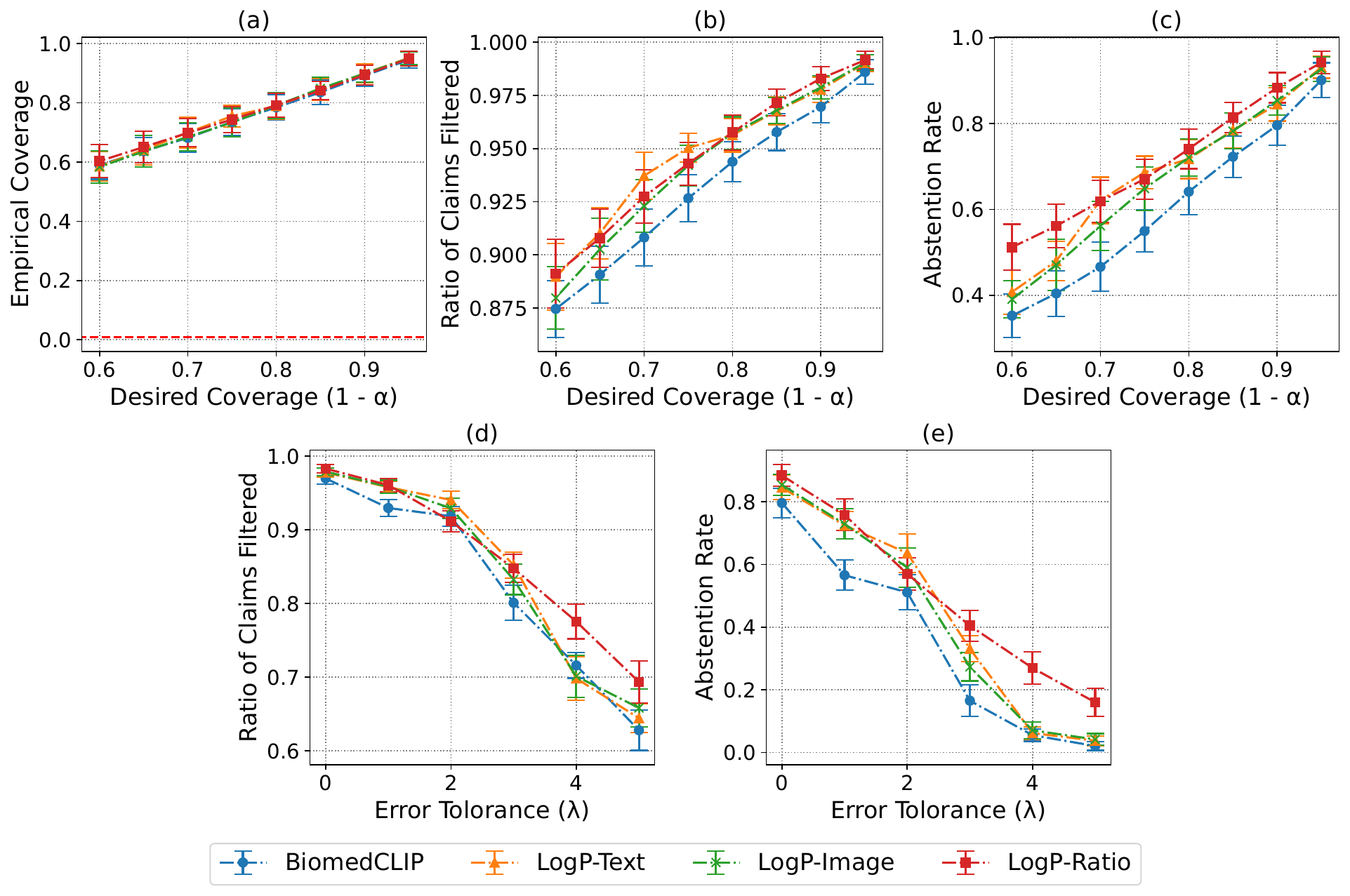}
    \caption{\modelname{MAIRA-2} results in \textit{medical report generation}.}
    \label{fig:mimic_merged}
\end{figure}

\subsection{Results}
Given limited space, we present the results of \modelname{MAIRA-2} and defer the results of other LVLMs to the Appendix.

We plot the empirical coverage versus desired coverage in Fig.~\ref{fig:mimic_merged}a.
The results verify that \methodName~can achieve tight error control in the revised responses after filtering. This is particularly important considering that medical report generation is a much more challenging task that requires precise control over the risk of output hallucination.

We plot the average ratio of claims filtered and abstention rate at various desired levels of coverage in Fig.~\ref{fig:mimic_merged}b and Fig.~\ref{fig:mimic_merged}c, respectively, given fixed $\lambda=0$.
We observe a trend similar to the scene understanding task, where the ratio of filtered claims and abstention rate increases with the desired coverage level. Notably, given a fixed coverage level, \modelname{BiomedCLIP} archives the lowest ratio of filtered claims and abstention rate among all scoring functions.
Fig.~\ref{fig:mimic_merged}d and Fig.~\ref{fig:mimic_merged}e present the results with varying error tolerance while keeping $\alpha=0.1$. In particular, changing $\lambda$ from $0$ to $3$ reduces the abstention rate by more than half, while still maintaining a relatively low error (compared to the median of the loss of unfiltered responses, which is $7$).

\begin{figure}
    \centering
\includegraphics[width=0.8\linewidth]{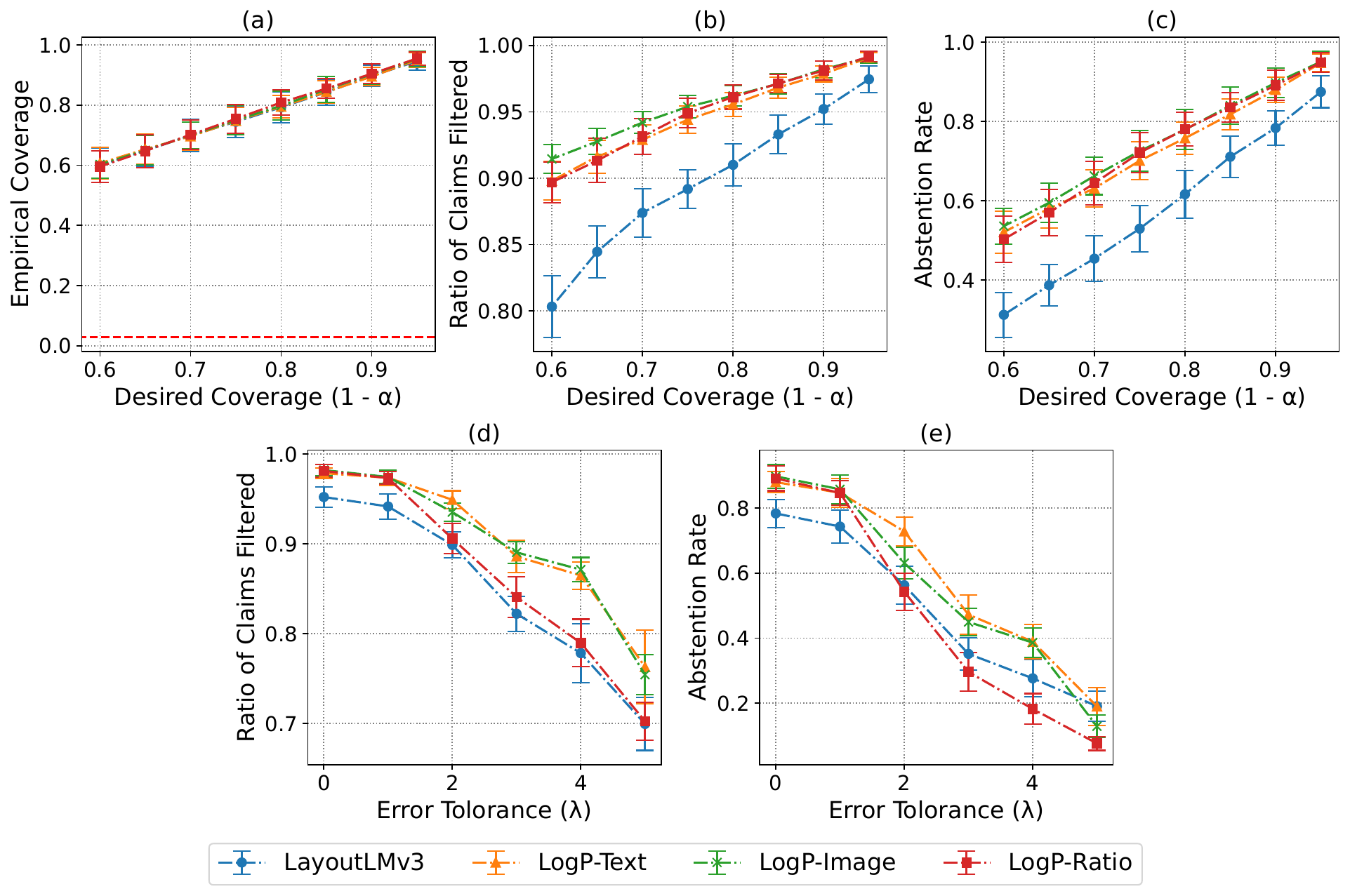}
    \caption{\modelname{LLaVA-NeXT} results in \textit{document understanding}.}
    \label{fig:sroie_merged}
\end{figure}

\section{Case Study III: Document Understanding}
\subsection{Setup}
Finally, we evaluate \methodName~on the document understanding task, where we randomly select $500$ invoice scan/images from the SROIE~\cite{huang2019icdar2019} dataset.

\paragraph{LVLMs.}
We consider two LVLMs, \modelname{LLaVA-Next}~\cite{liu2024llavanext}, which is the latest model in the \modelname{LLaVA} family with enhanced visual reasoning and OCR capabilities, and \modelname{Phi-3.5-vision-instruct}.

\paragraph{Error Annotation and Loss Function.}
We consider the following error types for this task:
(1) \textit{Field misinterpretation}: Incorrectly identify important fields such as mistaking "Subtotal" for "Total Amount", or misrecognizing non-existing fields.
(2) \textit{Numerical and quantitative errors}: Incorrect amounts, totals, or quantity values, as well as calculation discrepancies (e.g., subtotal, tax, and total relationship).
(3) \textit{Date error}: Misrecongizing date or misinterpreting date formats.
(4) \textit{Item error}: Misrecongizing item or item details, or falsely identifying non-existing items.
(5) \textit{Other errors}: Other errors such as misspelling or misrecognizing character, layout, and alignment issues.
Each occurrence of Numerical and Date Errors gets an additional loss of $3$, each occurrence of field or item error gets an additional loss of $2$, and the occurrence of other errors gets an additional loss of $1$. Similar to other tasks, we compute the accumulated loss for each response.

\paragraph{Scoring Function.}
We use \modelname{LayoutLMv3}~\cite{huang2022layoutlmv3} which is a pre-trained multi-modal transformer to derive embeddings for the claim text and document images and compute their cosine similarity as the external confidence score.

\subsection{Results}
We now show the results of \modelname{LLaVA-NeXT} and defer the results of \modelname{Phi-3.5-vision-instruct} to the Appendix.

We first verify the alignment of empirical coverage and desired coverage in Fig.~\ref{fig:sroie_merged}a. The results show that \methodName~can achieve the precise coverage on the document understanding task as desired.

We investigate the ratio of filtered claims and abstention rate with varying $\alpha$ and fixed $\lambda=1$ in Fig.~\ref{fig:sroie_merged}b and Fig.~\ref{fig:sroie_merged}c, respectively, and with varying $\lambda$ and fixed $\alpha=0.1$ in Fig.~\ref{fig:sroie_merged}d and Fig.~\ref{fig:sroie_merged}e, respectively. Besides the general trend that increasing desired coverage or reducing the error tolerance would result in filtering out more content and more frequently abstaining, we additionally observe that \modelname{LayoutLMv3} achieves significantly lower rates compared to other scoring functions based on the LVLM's internal confidence, e.g., preserving approximately $10\%$ more content when $\alpha=0.4$. This shows that a small dedicated model is more accurate than the prominent LVLMs in terms of verifying the factuality of claims regarding document images.

\section{Discussion}

\paragraph{Impact of Calibration Data Size.}
We have considered a fixed calibration data size of $400$.
To study the impact of calibration data size, we use \modelname{Llama-3.2-11B-Vision} as an example and vary the calibration data from $50$ to $400$ samples, each with $50$ random train-test splits, and plot the empirical coverage and ratio of claims filtered of three sets of $(\alpha, \lambda)$ parameters in Fig.~\ref{fig:pope_cali}.
We observe that \methodName~can consistently achieve the desired level of coverage while maintaining the same ratio of filtered claims regardless of the calibration data size, though a larger calibration dataset could help reduce the result variance.

\begin{figure}
    \centering
\includegraphics[width=0.8\linewidth]{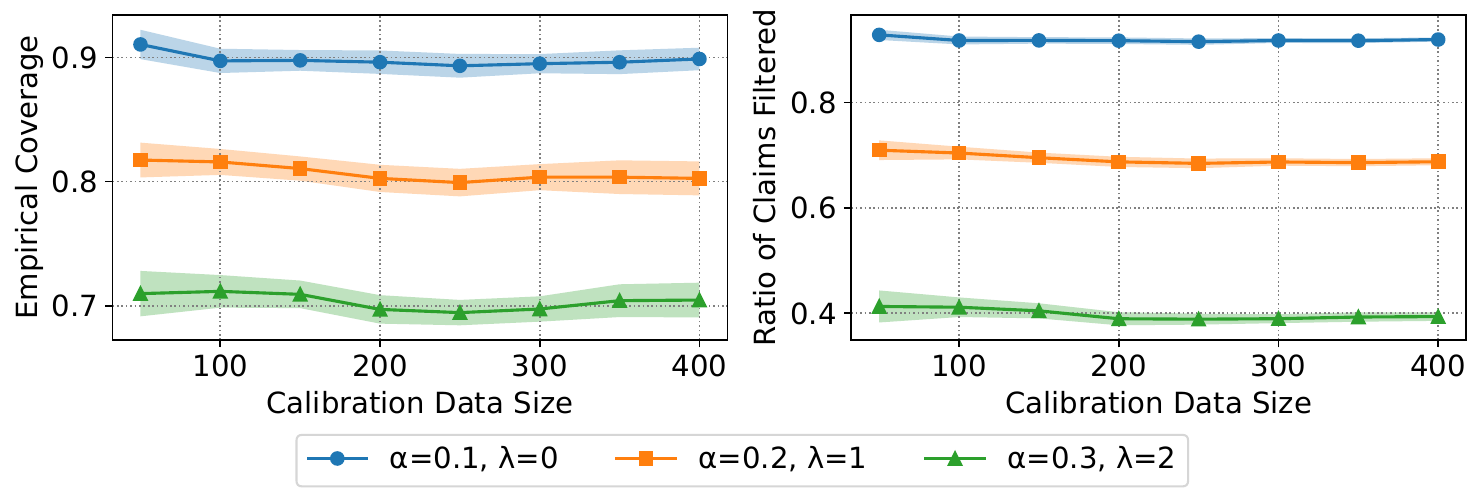}
    \caption{Impact of calibration data size on the empirical coverage and utility of \modelname{Llama-3.2-11B-Vision} on the \textit{scene understanding} task (with $95\%$ CI).}
    \label{fig:pope_cali}
\end{figure}

\paragraph{Discriminative vs. Generative Models.}
In our experiments, we observe that in many cases small discriminative vision-language models (e.g., \modelname{BiomedCLIP} for medical report generation and \modelname{LayoutLMv3} for document understanding) outperform LVLMs in terms of capturing the relevance between a text claim and an image. One potential reason is that compared to large generative models, small discriminative models are easier to optimize and can learn useful representations more efficiently. This hints that besides serving as the image encoder, these models can be used as critics to censor LVLM outputs and improve factuality with lower computational costs.

\paragraph{Coverage vs. Reliability.}
Although our method can achieve the precise coverage as specified by the user, the statistical guarantee only holds marginally with split conformal prediction. However, conditional guarantees may be required for certain applications, e.g., to ensure health equity among groups of patients in healthcare. Future work could consider the integration with advanced conformal methods to achieve conditional validity~\cite{gibbs2023conformal}.
Besides coverage, investigating other important aspects of LVLM reliability, such as omission, and designing better scoring functions to achieve the same level of coverage while preserving more content are also interesting avenues for future research.

\section{Conclusion}
In this work, we propose \methodName, a framework for achieving statistical factuality guarantee of LVLM output through decomposing responses into individual verifiable hypotheses and filtering out those with low confidence given the image content. We demonstrate with three application domains that by choosing the desired error rate and tolerance, \methodName~offers users flexible control over the hallucination risk of LVLM output.

\bibliographystyle{unsrt}  
\bibliography{reference}  %

\newpage
\appendix

\section{Implementation Details}

\paragraph{Prompts for Image-conditioned Free-text Generation.} We use the following prompt for evaluating LVLMs on the scene and document understanding tasks: \textit{``$\langle \texttt{Image} \rangle$ Provide a detailed description of the given image.''}.
To evaluate medical (radiology) report generation, we use the following prompt for \modelname{LLaVA-Med}: \textit{``$\langle \texttt{Image} \rangle$ What does the chest X-ray show?''}, whereas \modelname{CvT2DistilGPT2} and \modelname{MAIRA-2} do not require any text prompt for generating reports.

\paragraph{Prompts for Error Annotation.} We include our prompts for LLM-assisted error annotation of the scene understanding, medical report generation, and document understanding tasks in Table~\ref{tab:label_prompt}.

\begin{table*}[ht]
    \centering
    \caption{Prompts for error annotation.}
    \label{tab:label_prompt}
\resizebox{0.8\linewidth}{!}{
\begin{tcolorbox}[breakable,title=Error annotation prompt for scene understanding]
        \textbf{System} \textit{``You are an expert annotator tasked with evaluating statements generated by a vision-language model (VLM).\
Given an image and a claim, your task is to verify the factuality of the claim based on how well it aligns with the provided image.\
You should focus only on significant or material correctness, ignoring minor differences or non-essential details, especially in spatial relationships or specific object types.}

\textit{The errors are categorized as follows:
1. **Object Identification (Object)**: The claim involves hallucinated or wrongly identified objects. Ignore minor distinctions between similar objects (e.g., slotted spoon vs regular spoon) unless it fundamentally changes the meaning of the claim.
2. **Attribute Accuracy (Attribute)**: The claim involves incorrect attributes (e.g., color, size, shape). Only flag attributes if they are critical to the understanding of the claim.
3. **Spatial Relations (Spatial)**: The claim involves incorrect spatial relationships between objects. Only flag spatial errors if they significantly change the scene (e.g., "above the water" vs. "in the water" can be ignored unless the context requires precision).
4. **Interaction/Action Accuracy (Interaction)**: The claim involves incorrect or hallucinated action or interaction.
5. **Quantitative Information (Quantitative)**: The claim involves incorrect numeric details (e.g., wrong object count).
}

\textit{For each claim, generate a JSON object with four fields:
- "reasoning": a brief explanation of why the claim is correct or incorrect.
- "label": a boolean value (True or False) where True means the claim is factually correct, and False means it is incorrect.
- "error\textunderscore type": a list of error types (e.g., ["Object", "Attribute"]) if the claim contains errors, or an empty list if the claim is fully correct.
}

\textit{
Example:
Given an image of two orange cats, and the following list of claims:
1. This image features several cute cats.
2. There are a total number of three cats.
3. One cat is orange, the others are black.
4. There is also a dog behind the cats.
}

\textit{
Return:
[
    \{"reasoning": "The claim is general and no significant error can be found.", "label": true, "error\textunderscore type": []\},
    \{"reasoning": "There are two cats in the image, not three.", "label": false, "error\textunderscore type": ["Quantitative"]\},
    \{"reasoning": "One cat is orange, but the other is not black.", "label": false, "error\textunderscore type": ["Quantitative", "Attribute"]\},
    \{"reasoning": "There is no dog in the image.", "label": false, "error\textunderscore type": ["Object"]\}
]
''}

\textbf{User} \textit{``List of claims:
\{claims\}
For each claim, return a JSON object with "reasoning", "label" (true or false),
"error\textunderscore type" (might contain multiple types from ["Object", "Attribute", "Spatial", "Interaction", "Quantitative"]).
''}
        
    \end{tcolorbox}
    }

\resizebox{0.8\linewidth}{!}{
\begin{tcolorbox}[breakable,title=Error annotation prompt for medical report generation]
        \textbf{System} \textit{``You are an experienced radiologist tasked with evaluating statements generated by a Medical AI model.\
Given a chest x-ray image, a ground truth report generated by expert human radiologist, and a claim generated by the AI model, your task is to verify the factuality of the claim based on how well it aligns with the provided ground truth report. IMPORTANT: A claim should be deemed correct only if it is directly entailed by the ground truth report.}

\textit{The errors are categorized as follows:
1. **Conflicting Error (Conflicting)**: The claim directly contradicts information provided in the ground truth report.
2. **Implausible Error (Implausible)**: The claim does not directly conflict with or align with the ground truth report, and is implausible within the given context.
3. **Plausible Error (Plausible)**: The claim does not directly conflict with or align with the ground truth report, but remains plausible within the given context.
}

\textit{For each claim, generate a JSON object with four fields:
- "reasoning": a brief explanation of why the claim is correct or incorrect.
- "label": a boolean value (True or False) where True means the claim is factually correct, and False means it is incorrect.
- "error\textunderscore type": a list of error types (e.g., ["Conflicting", "Plausible"]) if the claim contains errors, or an empty list if the claim is fully correct.
}

\textit{
Example:
Given a chest x-ray report and the following list of claims:
1. There is no evidence of lung consolidation.
2. The heart size is mildly enlarged.
3. There are signs of a pleural effusion.
}

\textit{
Return:
[
    {"reasoning": "The ground truth report confirms no lung consolidation.", "label": true, "error\textunderscore type": []},
    {"reasoning": "The ground truth report describes the heart size as normal.", "label": false, "error\textunderscore type": ["Conflicting"]},
    {"reasoning": "The ground truth report does not mention a pleural effusion, but it is a plausible interpretation in some cases.", "label": false, "error\textunderscore type": ["Plausible"]}
]
''}

\textbf{User} \textit{``List of claims:
\{claims\}
For each claim, return a JSON object with "reasoning", "label" (true or false),
"error\textunderscore type" (might contain multiple types from ["Conflicting", "Implausible", "Plausible"]).
''}
        
    \end{tcolorbox}
    }

\resizebox{0.8\linewidth}{!}{
    \begin{tcolorbox}[breakable,title=Error annotation prompt for document understanding]
        \textbf{System} \textit{``You are an expert annotator tasked with evaluating statements generated by a Document AI model.\
Given a ground truth document (image and text) and a claim, your task is to verify the factuality of the claim based on how well it aligns with the provided document.
}

\textit{The errors are categorized as follows:
1. **Field Misinterpretation (Field)**: Incorrectly identify important fields such as mistaking "Invoice Date" for "Due Date", "Subtotal" for "Total Amount", or misrecognize non-existing field.
2. **Numerical and Quantitative Errors (Numerical)**: Incorrect amounts, totals, or quantity values, as well as calculation discrepancies (e.g., subtotal, tax, and total relationship).
3. **Date Error (Date)**: Misrecongizing date or misinterpreting date formats.
4. **Item Error (Item)**: Misrecongizing item or item details, or falsely identifying non-existing item.
5. **Other Errors (Other)**: Other errors such as misspell or misrecognize character, layout and alignment issues.
}

\textit{For each claim, generate a JSON object with four fields:
- "reasoning": a brief explanation of why the claim is correct or incorrect.
- "label": a boolean value (True or False) where True means the claim is factually correct, and False means it is incorrect.
- "error\textunderscore type": a list of error types (e.g., ["Numerical", "Item"]) if the claim contains errors, or an empty list if the claim is fully correct.
}

\textit{
Example:
Given an invoice of buying a Chopping Board at a shop named \"Walmart\", and the following list of claims:
1. This image is a printed invoice.
2. The merchant name is \"Wallmart\".
3. The items listed on the receipt include two Chopping Board, and a Knife.
}

\textit{
Return:
[
    {"reasoning": "The image shows a printed invoice.", "label": true, "error\textunderscore type": []},
    {"reasoning": "The merchant name is spelled incorrectly.", "label": false, "error\textunderscore type": ["Other"]},
    {"reasoning": "Only one Chopping Board, and no Knife purchased.", "label": false, "error\textunderscore type": ["Numerical", "Item"]},
]
''}

\textbf{User} \textit{``List of claims:
\{claims\}
For each claim, return a JSON object with "reasoning", "label" (true or false),\
"error\textunderscore type" (might contain multiple types from ["Field", "Numerical", "Date", "Item", "Other"]).
''}
        
    \end{tcolorbox}}
\end{table*}

\paragraph{Decomposition and Merge Operators.}
We implement the decomposition and merging operations by prompting the language model part of the LVLM. 
For rare cases where the LVLM does not support or cannot correctly implement the decomposing operation (e.g., dedicated models such as \modelname{MAIRA-2} and \modelname{CvT2DistilGPT2}), we use \modelname{GPT-4o-mini} as the substitute model to implement the decomposition and combination operations.
The prompt for decomposing claims is \textit{``Breakdown the above statement into a set of indepedent and self-contained claims. Each claim should be a short sentence. Output only a numbered list of claims.''}.
The prompt for merging claims is \textit{``Merge the above claims about an image into a cohesive statement. Reuse the words from the original claims and do not generate any new claims.''}.

\section{Additional Results}

\subsection{Annotation Reliability}
\paragraph{Human Raters.} We used \modelname{GPT-4o} for assisting with annotating error types. To verify its annotation quality, we randomly select a subset of $50$ images and $1,182$ associating claims generated by \modelname{LLaVA-1.5} on the scene understanding task, and recruit two human annotators to generate independent error type annotations.
The measured averaged Intraclass Correlation Coefficient (ICC) between \modelname{GPT-4o} and human annotations is $0.85$, with the $95\%$ CI being $[0.82, 0.87]$. This result confirmed that the annotations show high inter-rater reliability (by convention, any ICC value above 0.75 is considered to be good reliability~\cite{koo2016guideline}).

\paragraph{LVLM Raters.} We measured the averaged ICC between GPT-4o and Gemini-1.5-pro on the scene understanding task to be $0.81$, with the $95\%$ CI being $[0.77, 0.85]$, which shows high inter-rater reliability.

\subsection{LVLM Output Distribution}

In Fig.~\ref{fig:vlm_out_dist}, we compare the quality of vanilla LVLM outputs (i.e., raw responses without filtering any claim) by visualizing the distribution of the number of claims and loss per response.
An ideal LVLM should be expressive (output more claims) while maintaining a low risk of hallucination (yield low loss values).
On the scene understanding task, \modelname{GPT-4o-mini} clearly outperforms other models but is still prone to errors. On the medical report generation task, \modelname{CvT2DistilGPT2} and \modelname{MAIRA-2} both outperform \modelname{LLaVA-Med} by a large margin, but still have very high loss values in most responses.
Similarly, on the document understanding tasks, both model show a long tail in loss distribution.
These observations necessitate the adoption of error control methods with statistical guarantees.

\begin{figure}[t]
     \centering
     \begin{subfigure}[b]{0.7\linewidth}
         \centering
        \includegraphics[width=0.9\linewidth]{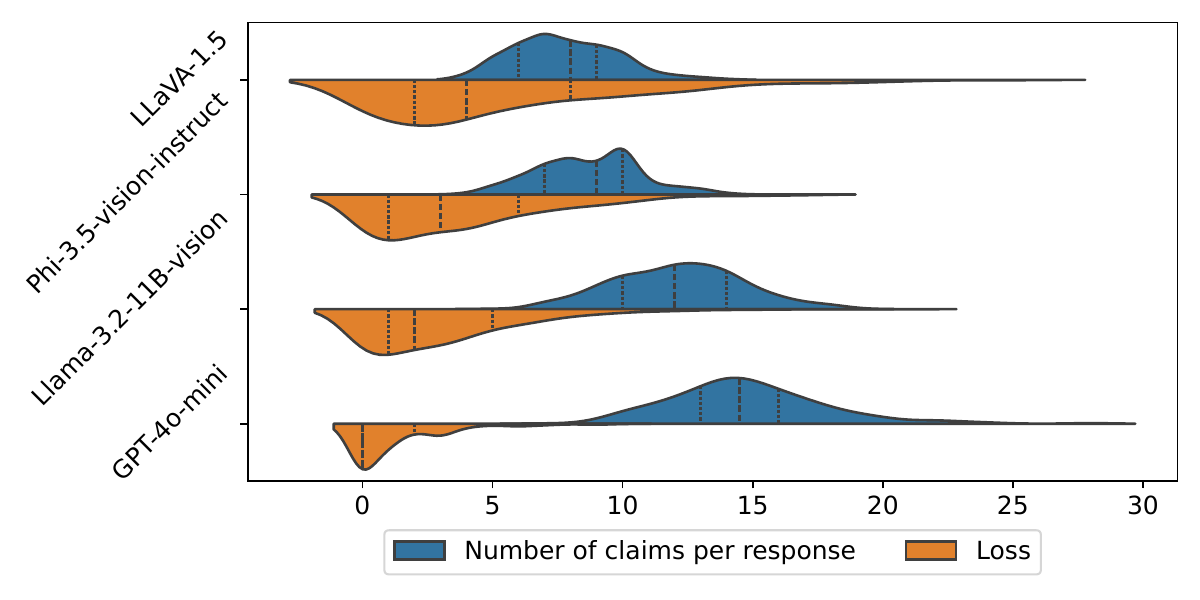}
         \caption{Scene understanding}
         \label{subfig:pope_vlm}
     \end{subfigure}

     \begin{subfigure}[b]{0.7\linewidth}
         \centering
        \includegraphics[width=0.9\linewidth]{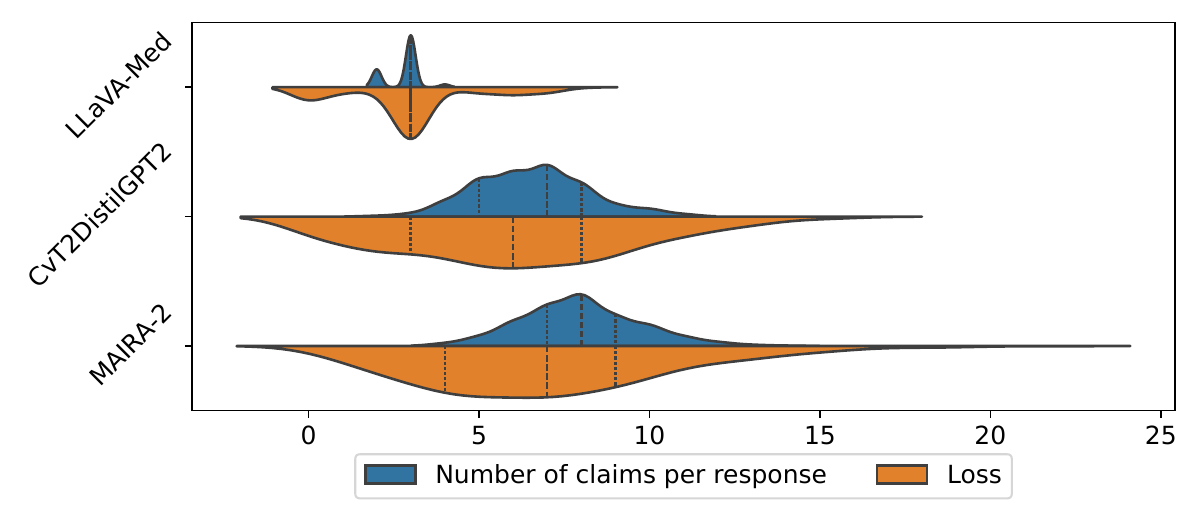}
         \caption{Medical report generation}
         \label{subfig:mimic_vlm}
     \end{subfigure}

     \begin{subfigure}[b]{0.7\linewidth}
         \centering
        \includegraphics[width=0.9\linewidth]{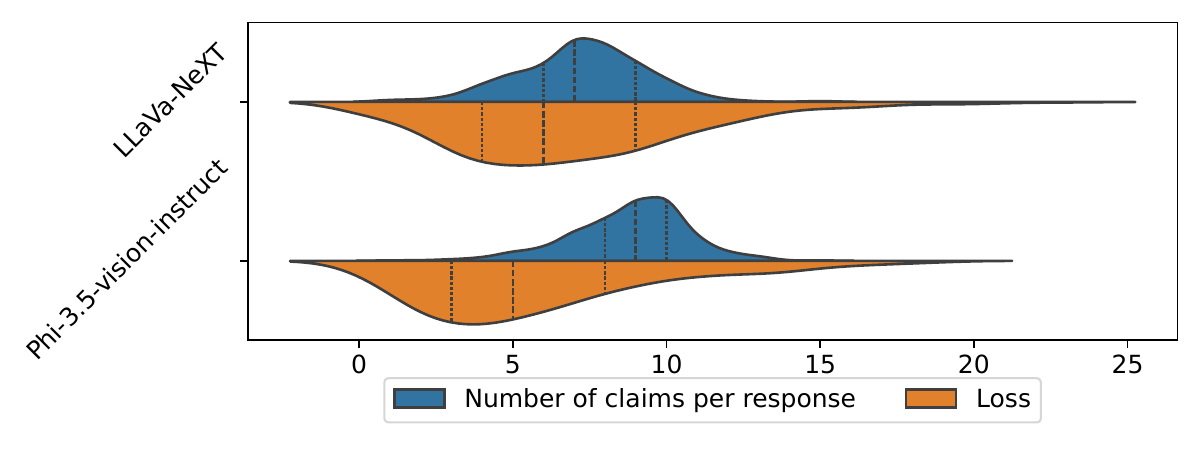}
         \caption{Document understanding}
         \label{subfig:sroie_vlm}
         
     \end{subfigure}
        \caption{Comparison of the quality of raw LVLM responses.}
        \label{fig:vlm_out_dist}
\end{figure}

\subsection{Comparison to Heuristic-based Mitigation}
Heuristic-based mitigation, such as Woodpecker~\cite{yin2023woodpecker}, relies on a series of external models, including BLIP-2, GroundingDINO, and proprietary models such as GPT, to reduce object hallucination, and thus it is not readily applicable to the specialized domains considered in our paper (i.e., medicine and finance).
Even if Woodpecker is applicable, unlike our method which is driven by a confidence score, Woodpecker is driven by the matching of textual claims to objects extracted from the image. In this sense, when Woodpecker fails to match an object, it filters out the claim. As such, our method offers a continuous confidence score that is tunable, whereas Woodpecker only offers a binary match/no match strategy.
We conducted additional experiments on the scene understanding task by randomly selecting $100$ images and measuring Woodpecker's claim filtering efficiency (in terms of TPR) and final response accuracy, with results shown in Table~\ref{tab:woodpecker}.
We observe that Woodpecker suffers from low TPR in claim filtering and low final response accuracy, partially due to GroundingDINO's high FPR in open-set object detection.

\begin{table}[h]
    \centering
    \caption{Comparision to heuristic-based mitigation.}
    \label{tab:woodpecker}
    \vspace{2mm}
    \resizebox{0.62\linewidth}{!}{
    \begin{tabular}{c|c|c}
    \toprule
                        & \textbf{Claim Filtering Efficiency (TPR) $\uparrow$} & \textbf{Response Accuracy $\uparrow$} \\ \midrule
    \textbf{Woodpecker} & 59.1\%                                    & 41\%                       \\ 
    \textbf{\textsc{ConfLVLM}}   & \textbf{95.3\%}                           & \textbf{90\%}              \\ \bottomrule
    \end{tabular}
    }

\end{table}

\subsection{Omitted Results from Main Paper}
Here we provide the omitted results from main paper for medical report generation and document understanding.

\paragraph{Medical Report Generation.}
Fig.~\ref{fig:mimic_coverage} compares the empirical and desired coverage of \modelname{LlaVa-Med}, \modelname{CvT2DistilGPT2}, and \modelname{MAIRA-2} on the medical report generation task.
The same conclusion is drawn as in the general scene understanding setting: \methodName~
achieves the desired level of coverage across all types of scoring functions, whereas Vanilla LVLM (i.e., responses without any filtration) produces significantly low coverage. 
Fig.~\ref{fig:mimic_cont} shows the average ratio of filtered claims  across a range of desired coverage, whereas Fig.~\ref{fig:mimic_abstention} presents the abstention rate as a function of desired coverage.

\begin{figure}[t]
    \centering
\includegraphics[width=0.86\linewidth]{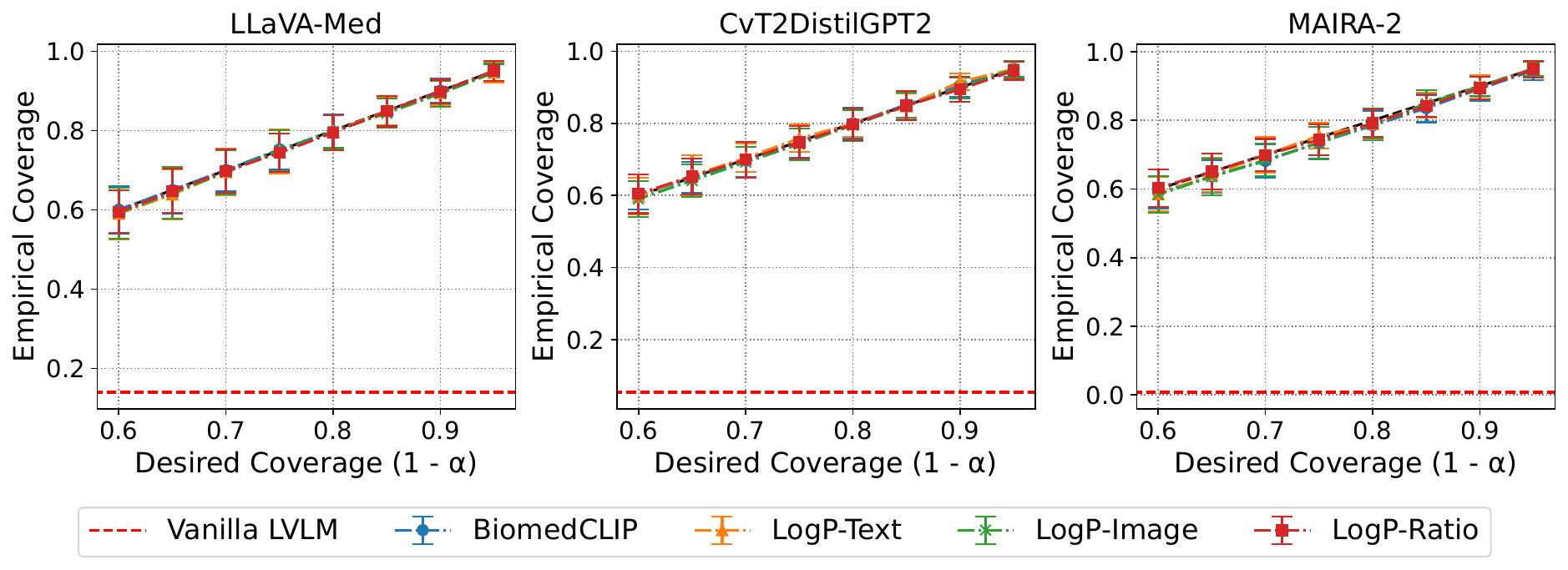}
    \caption{Alignment between the empirical and desired (theoretical) coverage on the \textit{medical report generation} task (with $\lambda=0$). Vanilla LVLM (red dashed line) refers to the base setting where the LVLM-generated responses are returned to users without using \methodName.}
    \label{fig:mimic_coverage}
\end{figure}

\begin{figure}[t]
    \centering
\includegraphics[width=0.86\linewidth]{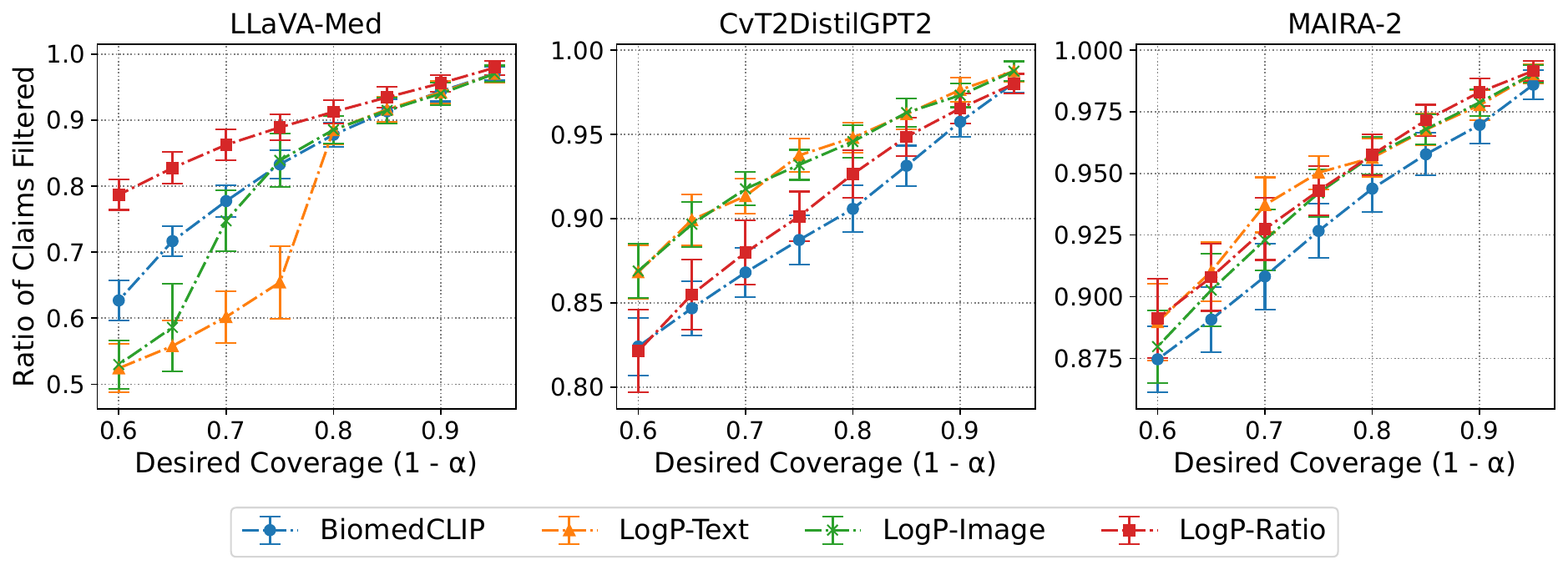}
    \caption{Average ratio of claims filtered with varying coverage using different scoring functions on the \textit{medical report generation} task (with $\lambda=0$).}
    \label{fig:mimic_cont}
\end{figure}

\begin{figure}[t]
    \centering
\includegraphics[width=0.86\linewidth]{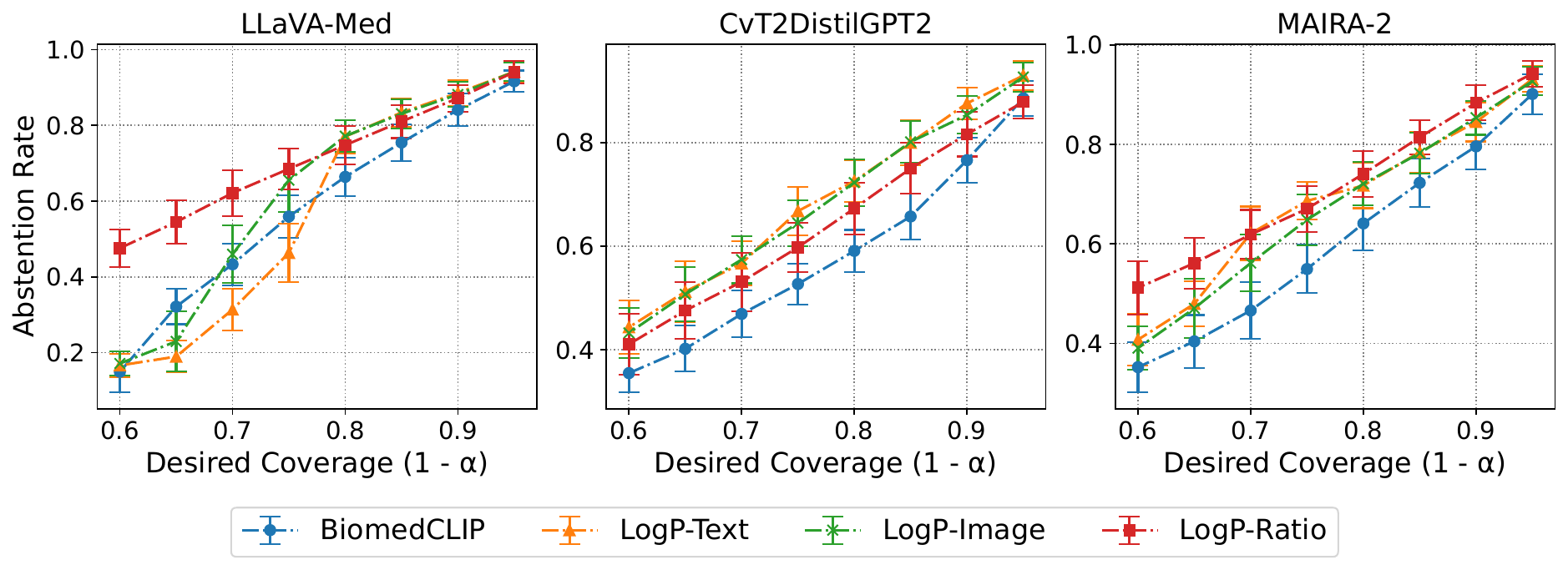}
    \caption{Abstention rate with varying coverage using different scoring functions on the \textit{medical report generation} task (with $\lambda=0$).}
    \label{fig:mimic_abstention}
\end{figure}

\begin{figure}[t]
    \centering
\includegraphics[width=0.8\linewidth]{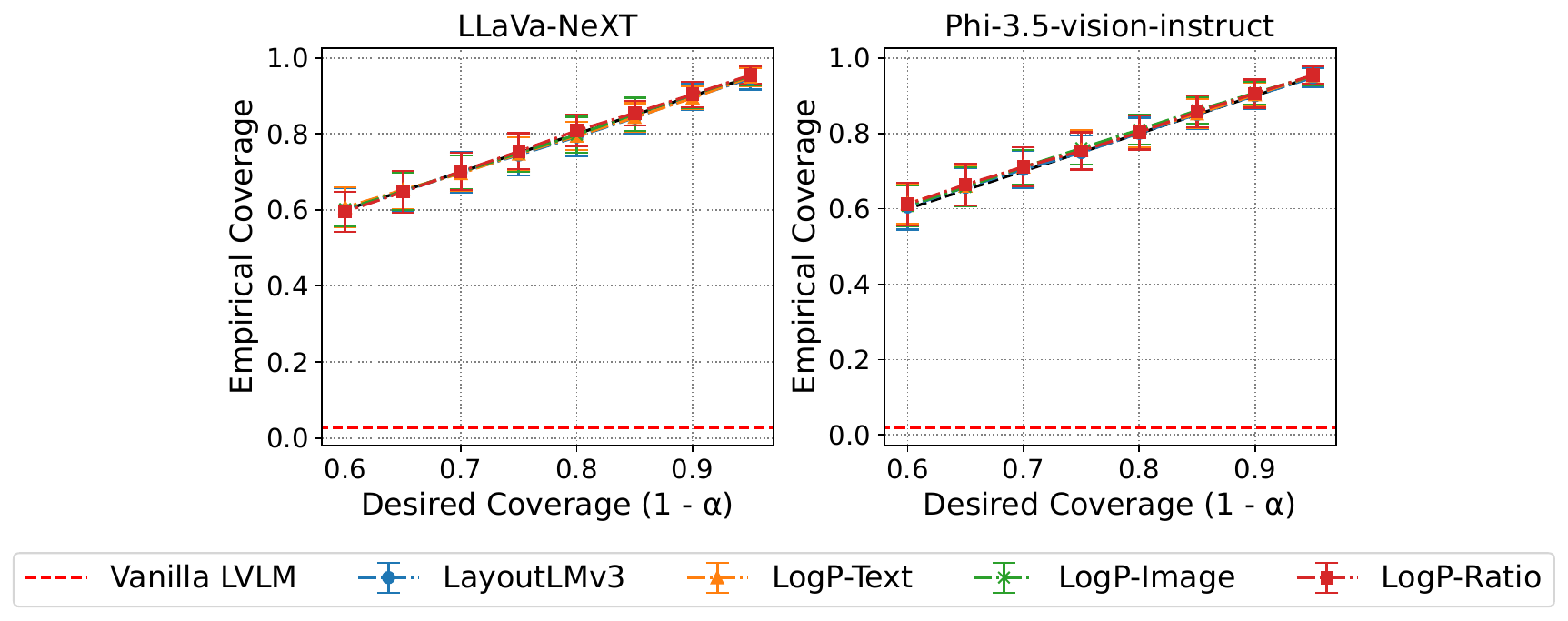}
    \caption{Alignment between the empirical and desired (theoretical) coverage on the \textit{document understanding} task (with $\lambda=0$). Vanilla LVLM (red dashed line) refers to the base setting where the LVLM-generated responses are returned to users without using \methodName.}
    \label{fig:sroie_coverage}
\end{figure}

\paragraph{Document Understanding.}
Fig.~\ref{fig:sroie_coverage} shows the alignment between the empirical and desired coverage of \modelname{LLaVA-Next} and \modelname{Phi-3.5-vision-instruct} using \methodName~on the document understanding task. The same conclusions regarding model coverage are reached as in the other two image understanding settings.
Fig.~\ref{fig:sroie_cont} shows the average ratio of filtered claims  across a range of desired coverage, whereas
Fig.~\ref{fig:sroie_abstention} presents the abstention rate as a function of desired coverage.
In Fig.~\ref{fig:sroie_err}, we compare \modelname{LLaVa-NeXT}'s response with different error tolerances and a fixed error rate.

\begin{figure}[t]
    \centering
\includegraphics[width=0.7\linewidth]{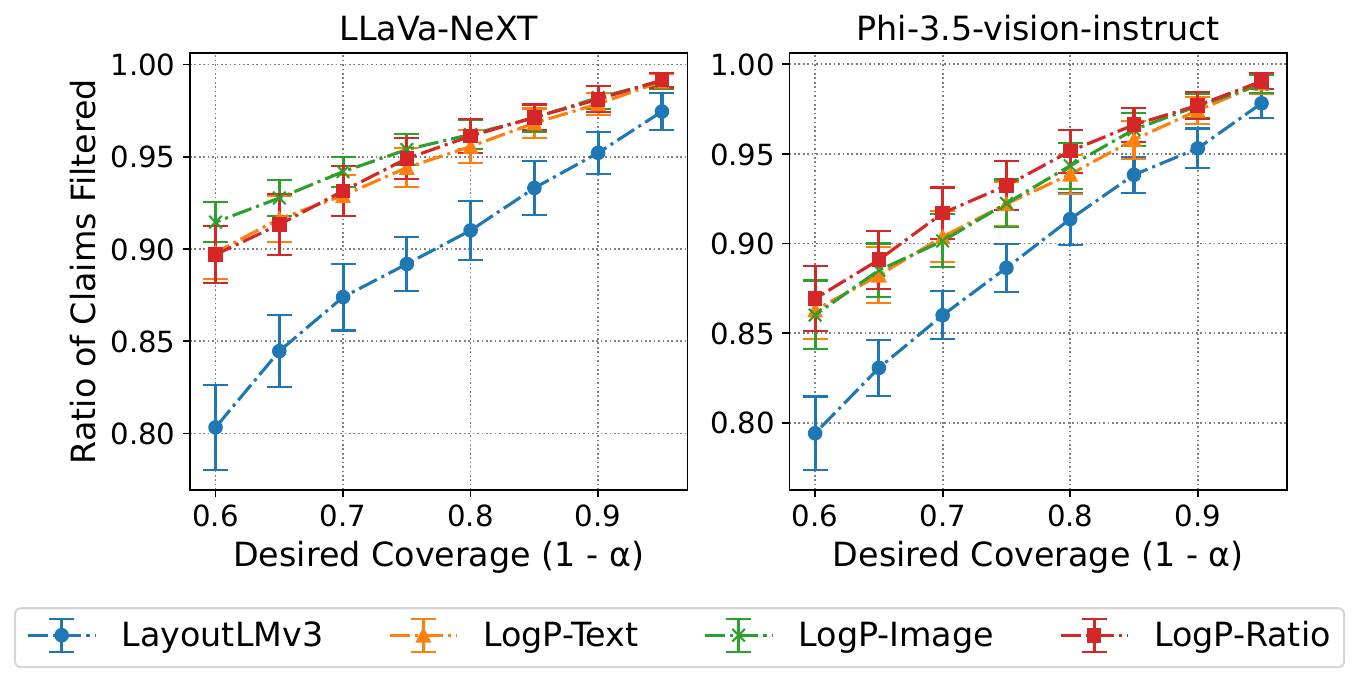}
    \caption{Average ratio of claims filtered with varying coverage using different scoring functions on the \textit{document understanding} task (with $\lambda=0$).}
    \label{fig:sroie_cont}
\end{figure}

\begin{figure}[t]
    \centering
\includegraphics[width=0.7\linewidth]{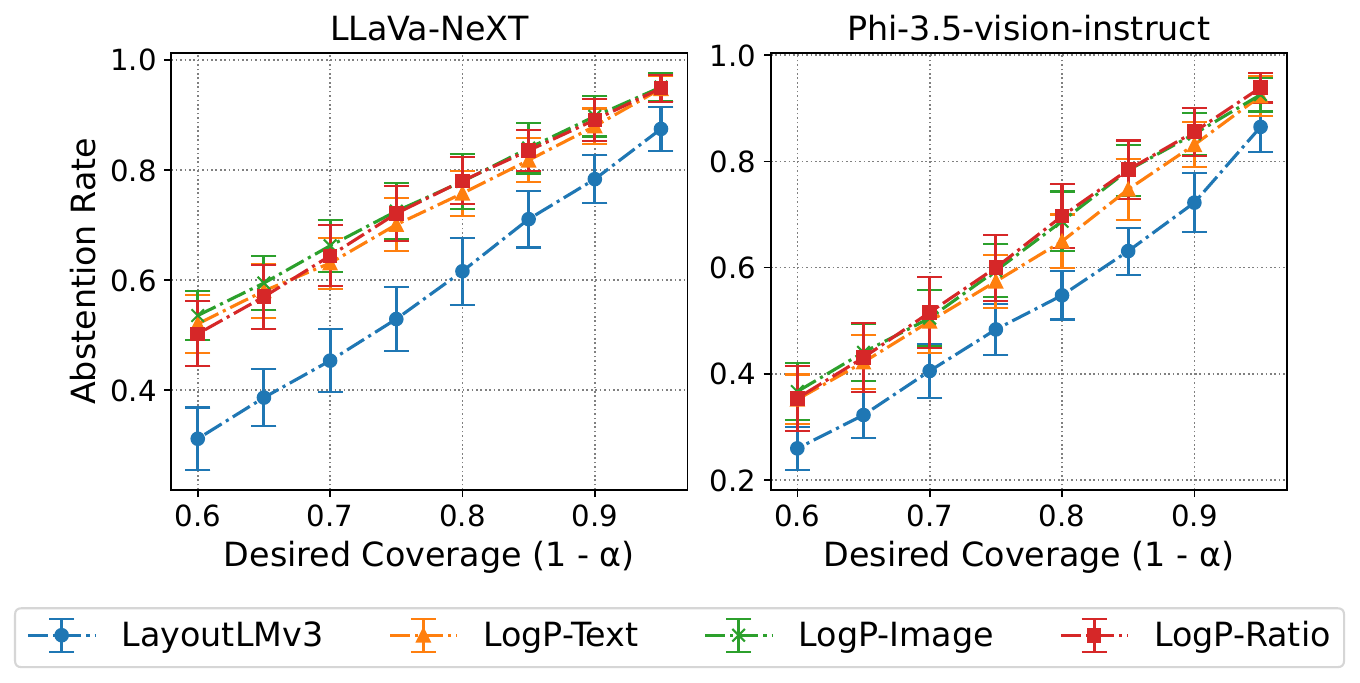}
    \caption{Abstention rate with varying coverage using different scoring functions on the \textit{document understanding} task (with $\lambda=0$).}
    \label{fig:sroie_abstention}
\end{figure}

\begin{figure}[t]
    \centering
    \includegraphics[width=0.7\linewidth]{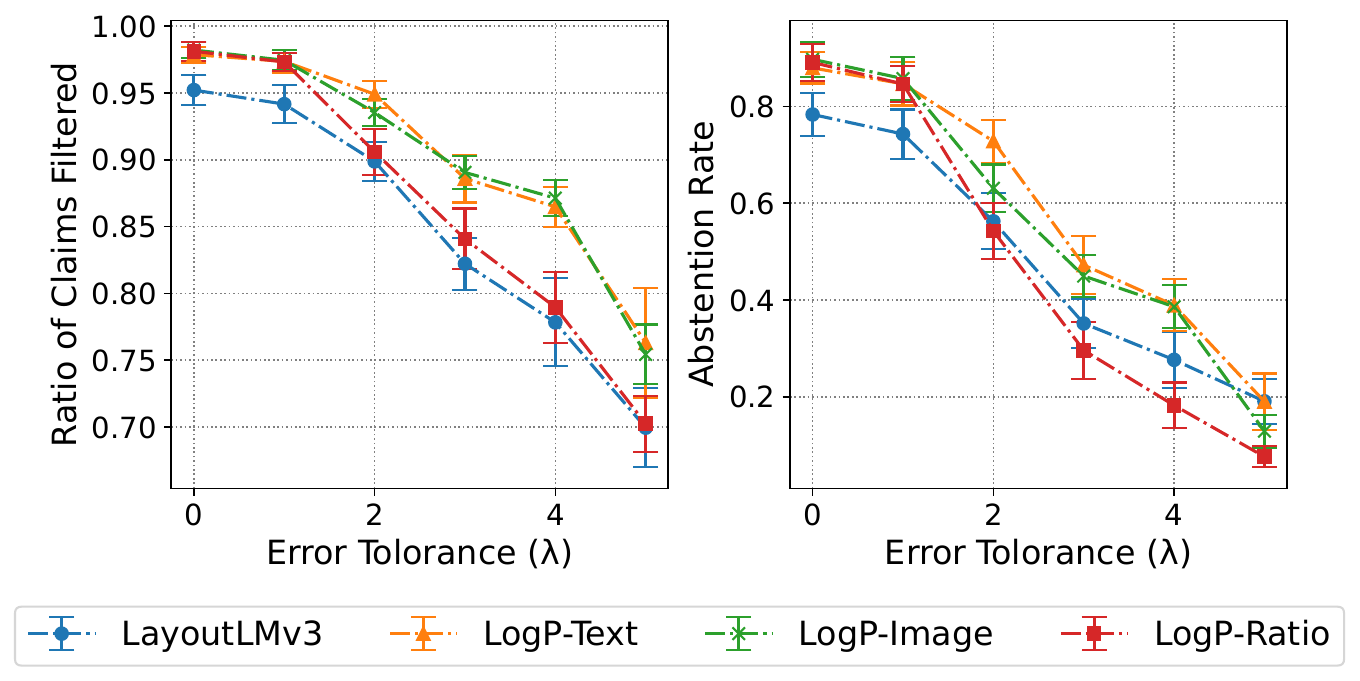}
    \caption{Comparison of \modelname{LLaVa-NeXT}'s response with different error tolerances ($\lambda$) while fixing $\alpha=0.1$ on the \textit{document understanding} task.}
    \label{fig:sroie_err}
\end{figure}

\subsection{Response-level and Claim-level Results}
In Table~\ref{tab:pope_res_extend}, we show the response-level (rate of responses containing at least one error and the average loss per response) and claim-level (TPR, FNR, and F1 for detecting erroneous responses) results under various $\alpha, \lambda$ configurations on the scene understanding task. Results are averaged over $50$ random data splits.

\begin{table}[t]
    \centering
    \resizebox{0.78\linewidth}{!}{
    \begin{tabular}{l|c|cc|ccc}
\hline
                                          &                                             & \multicolumn{2}{c|}{\textbf{Response-level}}                    & \multicolumn{3}{c}{\textbf{Claim-level}}                                                         \\ \cline{3-7} 
\multirow{-2}{*}{\textbf{LVLM}}           & \multirow{-2}{*}{\textbf{Configuration}}    & \textbf{Error Rate}            & \textbf{Average Loss}          & \textbf{TPR}                   & \textbf{FNR}                   & \textbf{F1}                    \\ \hline
                                          & Vanilla                                     & 0.8782                         & 5.5028                         & 0.0                            & 1.0                            & 0.0                            \\
                                          & \cellcolor[HTML]{EFEFEF}$\alpha=0.1, \lambda=0$ & \cellcolor[HTML]{EFEFEF}0.102 & \cellcolor[HTML]{EFEFEF}0.206 & \cellcolor[HTML]{EFEFEF}0.953 & \cellcolor[HTML]{EFEFEF}0.047 & \cellcolor[HTML]{EFEFEF}0.504 \\
                                          & \cellcolor[HTML]{EFEFEF}$\alpha=0.1, \lambda=2$ & \cellcolor[HTML]{EFEFEF}0.212 & \cellcolor[HTML]{EFEFEF}0.529 & \cellcolor[HTML]{EFEFEF}0.886 & \cellcolor[HTML]{EFEFEF}0.114 & \cellcolor[HTML]{EFEFEF}0.507 \\
                                          & \cellcolor[HTML]{EFEFEF}$\alpha=0.3, \lambda=0$ & \cellcolor[HTML]{EFEFEF}0.291 & \cellcolor[HTML]{EFEFEF}0.780 & \cellcolor[HTML]{EFEFEF}0.836 & \cellcolor[HTML]{EFEFEF}0.164 & \cellcolor[HTML]{EFEFEF}0.503 \\
\multirow{-5}{*}{\modelname{LLaVA-1.5}}               & \cellcolor[HTML]{EFEFEF}$\alpha=0.3, \lambda=2$ & \cellcolor[HTML]{EFEFEF}0.536 & \cellcolor[HTML]{EFEFEF}1.771 & \cellcolor[HTML]{EFEFEF}0.637 & \cellcolor[HTML]{EFEFEF}0.363 & \cellcolor[HTML]{EFEFEF}0.499 \\ \hline
                                          & Vanilla                                     & 0.850                         & 3.895                         & 0.0                            & 1.0                            & 0.0                            \\
                                          & \cellcolor[HTML]{EFEFEF}$\alpha=0.1, \lambda=0$ & \cellcolor[HTML]{EFEFEF}0.094 & \cellcolor[HTML]{EFEFEF}0.147 & \cellcolor[HTML]{EFEFEF}0.945 & \cellcolor[HTML]{EFEFEF}0.055 & \cellcolor[HTML]{EFEFEF}0.401 \\
                                          & \cellcolor[HTML]{EFEFEF}$\alpha=0.1, \lambda=2$ & \cellcolor[HTML]{EFEFEF}0.288 & \cellcolor[HTML]{EFEFEF}0.611 & \cellcolor[HTML]{EFEFEF}0.829 & \cellcolor[HTML]{EFEFEF}0.171 & \cellcolor[HTML]{EFEFEF}0.392 \\
                                          & \cellcolor[HTML]{EFEFEF}$\alpha=0.3, \lambda=0$ & \cellcolor[HTML]{EFEFEF}0.306 & \cellcolor[HTML]{EFEFEF}0.655 & \cellcolor[HTML]{EFEFEF}0.818 & \cellcolor[HTML]{EFEFEF}0.182 & \cellcolor[HTML]{EFEFEF}0.390 \\
\multirow{-5}{*}{\modelname{Phi-3.5-vision-instruct}} & \cellcolor[HTML]{EFEFEF}$\alpha=0.3, \lambda=2$ & \cellcolor[HTML]{EFEFEF}0.648 & \cellcolor[HTML]{EFEFEF}1.911 & \cellcolor[HTML]{EFEFEF}0.480 & \cellcolor[HTML]{EFEFEF}0.520 & \cellcolor[HTML]{EFEFEF}0.345 \\ \hline
                                          & Vanilla                                     & 0.793                         & 3.129                         & 0.0                            & 1.0                            & 0.0                            \\
                                          & \cellcolor[HTML]{EFEFEF}$\alpha=0.1, \lambda=0$ & \cellcolor[HTML]{EFEFEF}0.105 & \cellcolor[HTML]{EFEFEF}0.205 & \cellcolor[HTML]{EFEFEF}0.936 & \cellcolor[HTML]{EFEFEF}0.064 & \cellcolor[HTML]{EFEFEF}0.269 \\
                                          & \cellcolor[HTML]{EFEFEF}$\alpha=0.1, \lambda=2$ & \cellcolor[HTML]{EFEFEF}0.232 & \cellcolor[HTML]{EFEFEF}0.543 & \cellcolor[HTML]{EFEFEF}0.831 & \cellcolor[HTML]{EFEFEF}0.169 & \cellcolor[HTML]{EFEFEF}0.266 \\
                                          & \cellcolor[HTML]{EFEFEF}$\alpha=0.3, \lambda=0$ & \cellcolor[HTML]{EFEFEF}0.306 & \cellcolor[HTML]{EFEFEF}0.714  & \cellcolor[HTML]{EFEFEF}0.772 & \cellcolor[HTML]{EFEFEF}0.228 & \cellcolor[HTML]{EFEFEF}0.264 \\
\multirow{-5}{*}{\modelname{Llama-3.2-11B-vision}}    & \cellcolor[HTML]{EFEFEF}$\alpha=0.3, \lambda=2$ & \cellcolor[HTML]{EFEFEF}0.628 & \cellcolor[HTML]{EFEFEF}1.805 & \cellcolor[HTML]{EFEFEF}0.408 & \cellcolor[HTML]{EFEFEF}0.592 & \cellcolor[HTML]{EFEFEF}0.227 \\ \hline
                                          & Vanilla                                     & 0.493                         & 1.265                         & 0.0                            & 1.0                            & 0.0                            \\
                                          & \cellcolor[HTML]{EFEFEF}$\alpha=0.1, \lambda=0$ & \cellcolor[HTML]{EFEFEF}0.097 & \cellcolor[HTML]{EFEFEF}0.168 & \cellcolor[HTML]{EFEFEF}0.850 & \cellcolor[HTML]{EFEFEF}0.150 & \cellcolor[HTML]{EFEFEF}0.100 \\
                                          & \cellcolor[HTML]{EFEFEF}$\alpha=0.1, \lambda=2$ & \cellcolor[HTML]{EFEFEF}0.285 & \cellcolor[HTML]{EFEFEF}0.552 & \cellcolor[HTML]{EFEFEF}0.544 & \cellcolor[HTML]{EFEFEF}0.456 & \cellcolor[HTML]{EFEFEF}0.099 \\
                                          & \cellcolor[HTML]{EFEFEF}$\alpha=0.3, \lambda=0$ & \cellcolor[HTML]{EFEFEF}0.300 & \cellcolor[HTML]{EFEFEF}0.589 & \cellcolor[HTML]{EFEFEF}0.510 & \cellcolor[HTML]{EFEFEF}0.490 & \cellcolor[HTML]{EFEFEF}0.0988 \\
\multirow{-5}{*}{\modelname{GPT-4o-mini}}             & \cellcolor[HTML]{EFEFEF}$\alpha=0.3, \lambda=2$ & \cellcolor[HTML]{EFEFEF}0.493 & \cellcolor[HTML]{EFEFEF}1.265 & \cellcolor[HTML]{EFEFEF}0.0    & \cellcolor[HTML]{EFEFEF}1.0    & \cellcolor[HTML]{EFEFEF}0.0    \\ \hline
\end{tabular}
    }
    \vspace{2mm}
    \caption{Response- and claim-level results on the \textit{scene understanding} task (averaged over $50$ random splits).}
    \label{tab:pope_res_extend}
\end{table}

\subsection{Empirical Coverage and Utility with $\lambda > 0$}

In addition to the plots in the main paper with $\lambda = 0$, we plot the empirical coverage, ratio of claims filtered, and abstention rate with error tolerance $\lambda = 1,2$ for scene understanding in Fig.~\ref{fig:pope_coverage_lambda}, Fig.~\ref{fig:pope_cont_lambda}, and Fig.~\ref{fig:pope_abstent_lambda},
for medical report generation in Fig.~\ref{fig:mimic_coverage_lambda}, Fig.~\ref{fig:mimic_cont_lambda}, and Fig.~\ref{fig:mimic_abstent_lambda}, and for document understanding in Fig.~\ref{fig:sroie_coverage_lambda}, Fig.~\ref{fig:sroie_cont_lambda}, and Fig.~\ref{fig:sroie_abstent_lambda}, respectively.

\begin{figure}[t]
    \centering
    \begin{subfigure}[b]{\linewidth}
        \centering
        \includegraphics[width=0.6\linewidth]{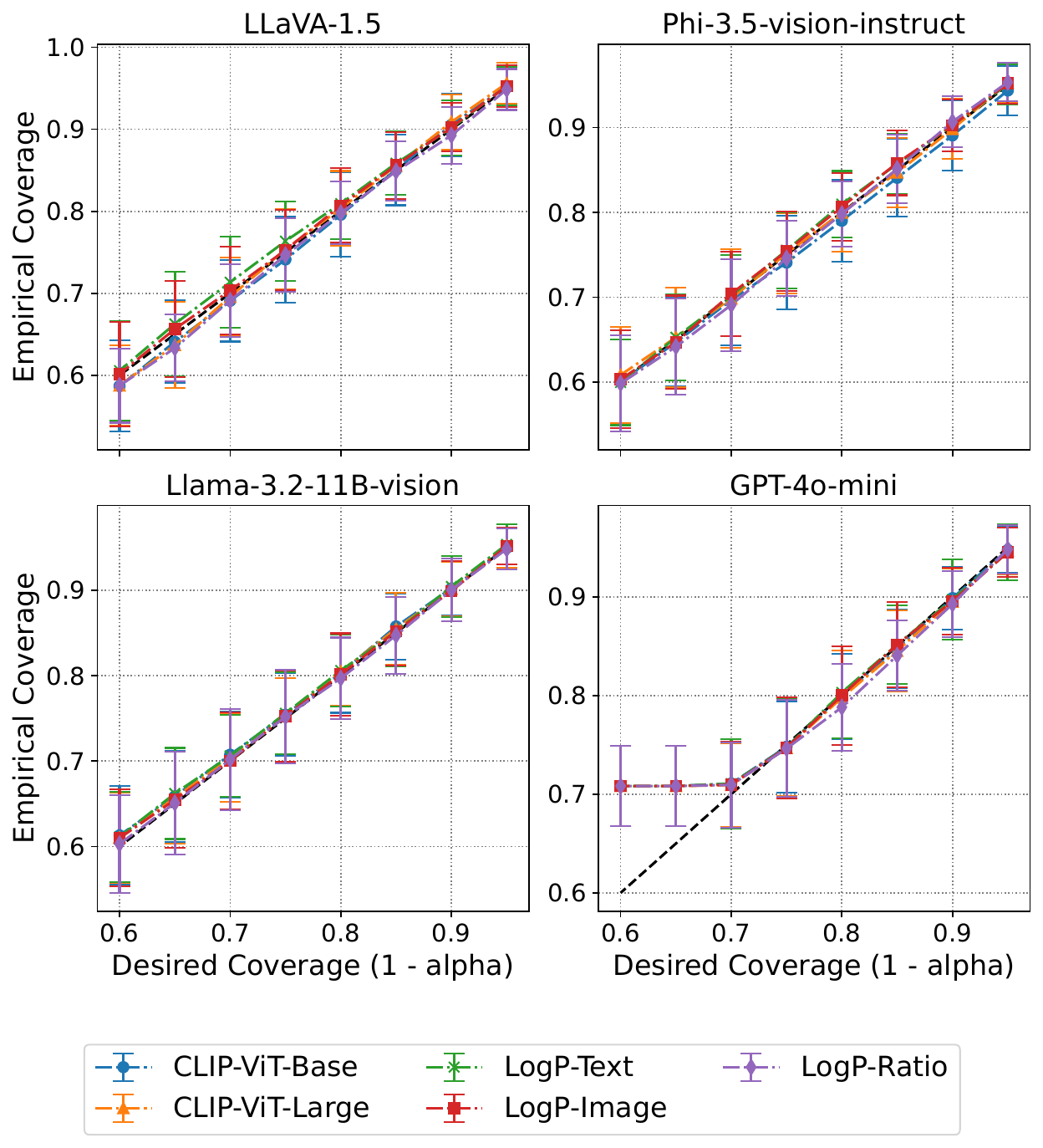}
        \caption{$\lambda=1$}
    \end{subfigure}

    \begin{subfigure}[b]{\linewidth}
        \centering
        \includegraphics[width=0.6\linewidth]{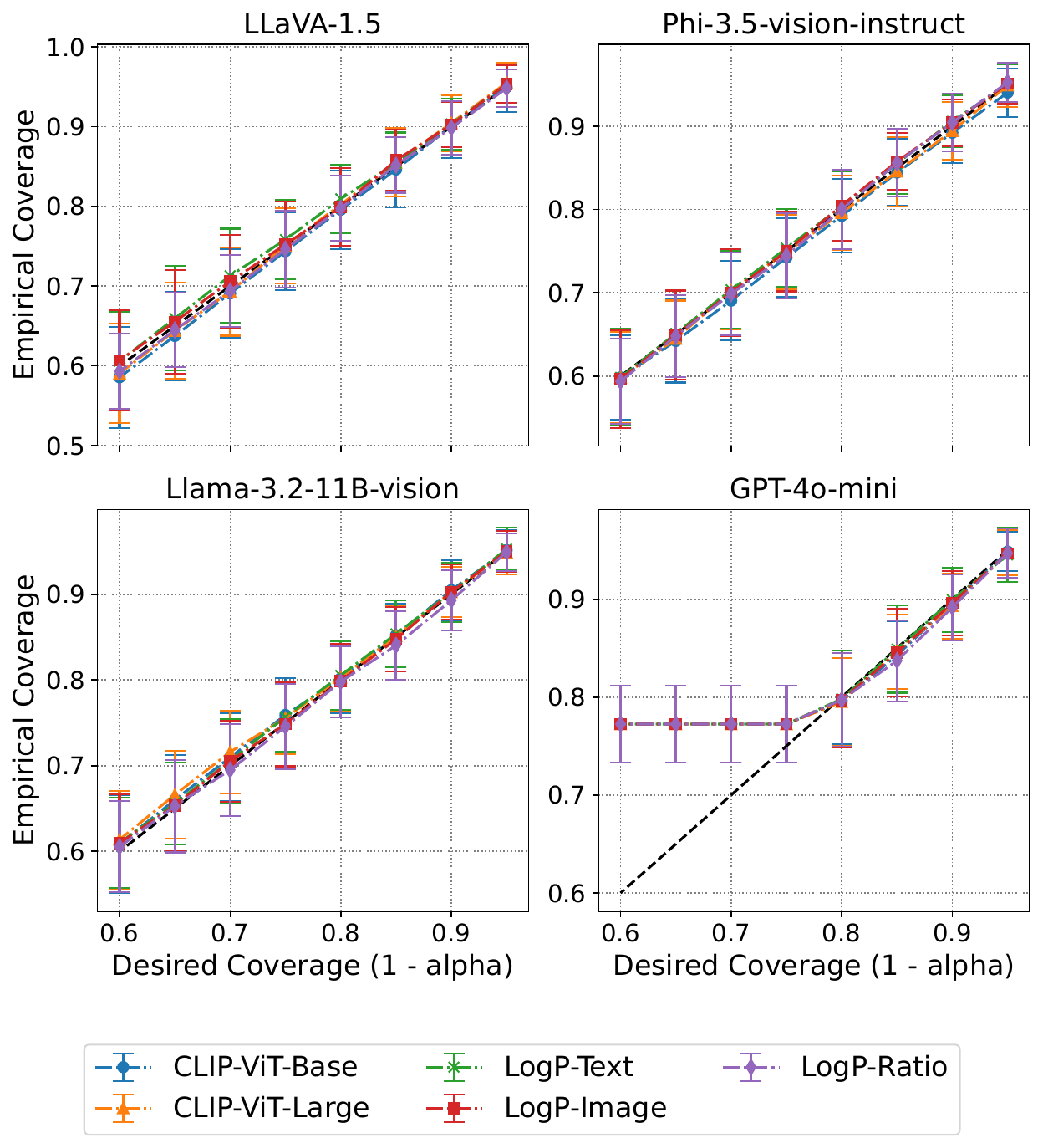}
        \caption{$\lambda=2$}
    \end{subfigure}

    \caption{Comparison of empirical and desired (theoretical) coverage on the scene understanding task with different error tolerances ($\lambda$).}
    \label{fig:pope_coverage_lambda}
\end{figure}

\begin{figure}[t]
    \centering
    \begin{subfigure}[b]{\linewidth}
        \centering
        \includegraphics[width=0.6\linewidth]{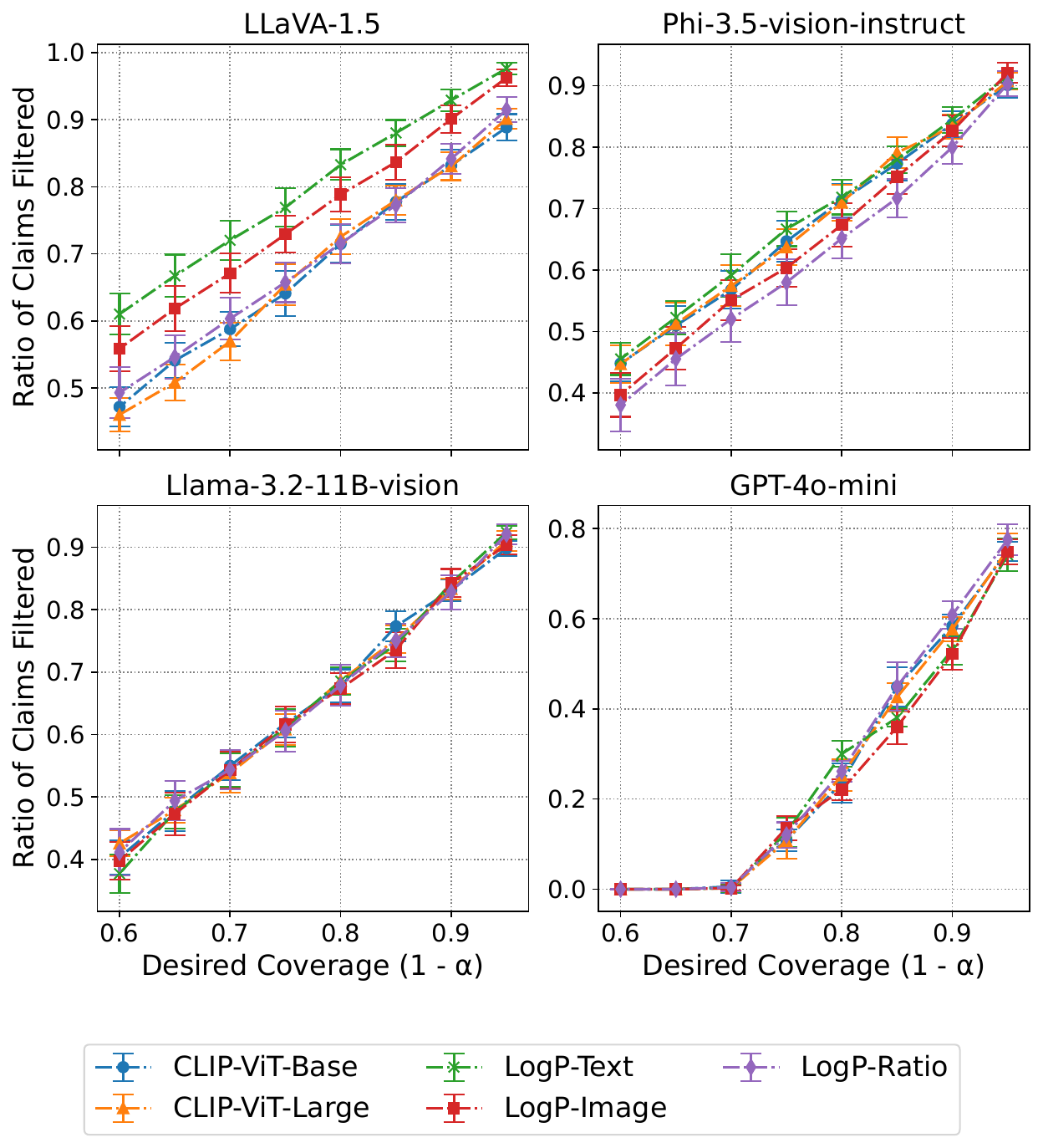}
        \caption{$\lambda=1$}
    \end{subfigure}

    \begin{subfigure}[b]{\linewidth}
        \centering
        \includegraphics[width=0.6\linewidth]{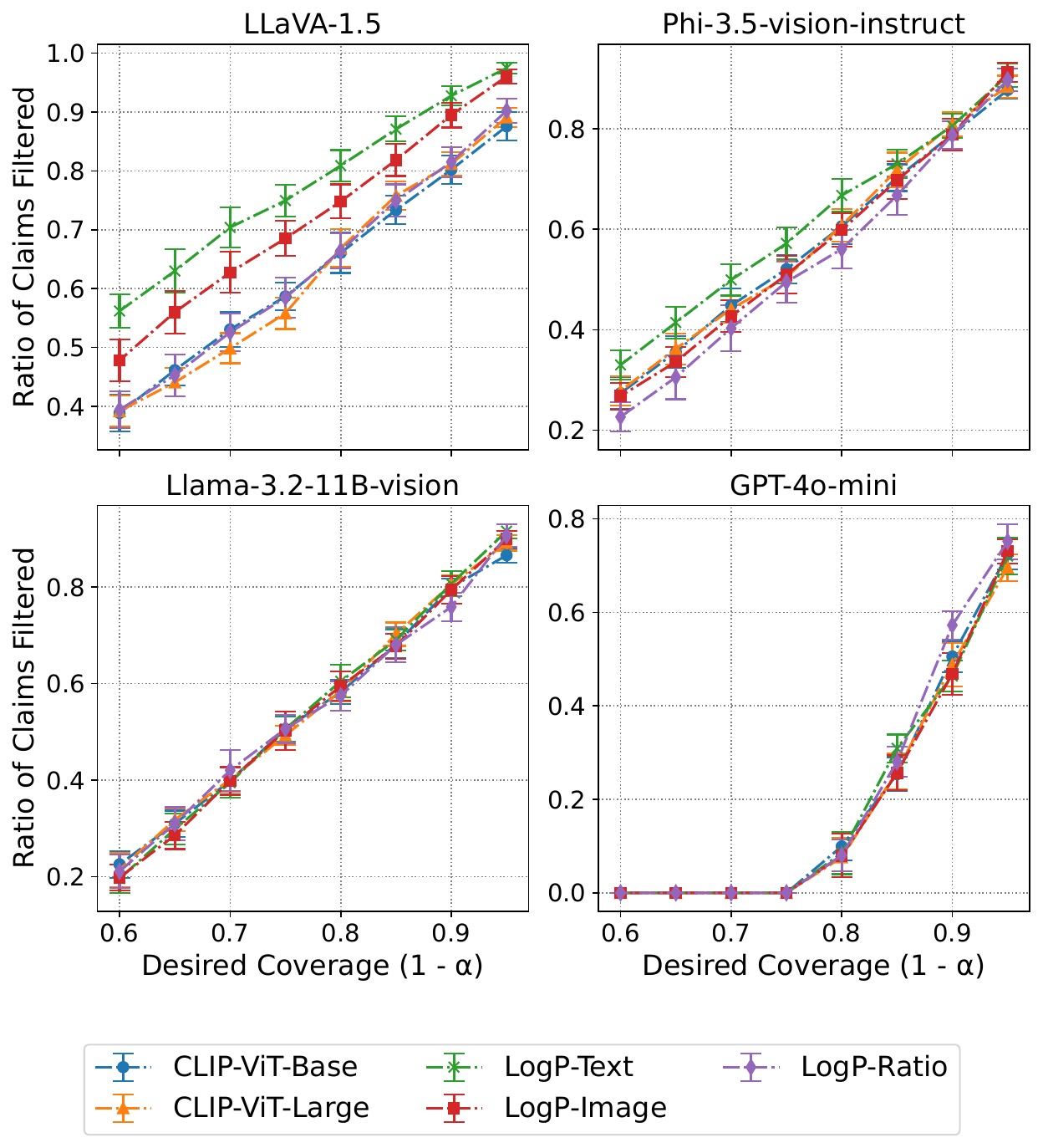}
        \caption{$\lambda=2$}
    \end{subfigure}

    \caption{Average ratio of claims filtered with varying coverage using different scoring functions on the scene understanding task with different error tolerances ($\lambda$).}
    \label{fig:pope_cont_lambda}
\end{figure}

\begin{figure}[t]
    \centering
    \begin{subfigure}[b]{\linewidth}
        \centering
    \includegraphics[width=0.6\linewidth]{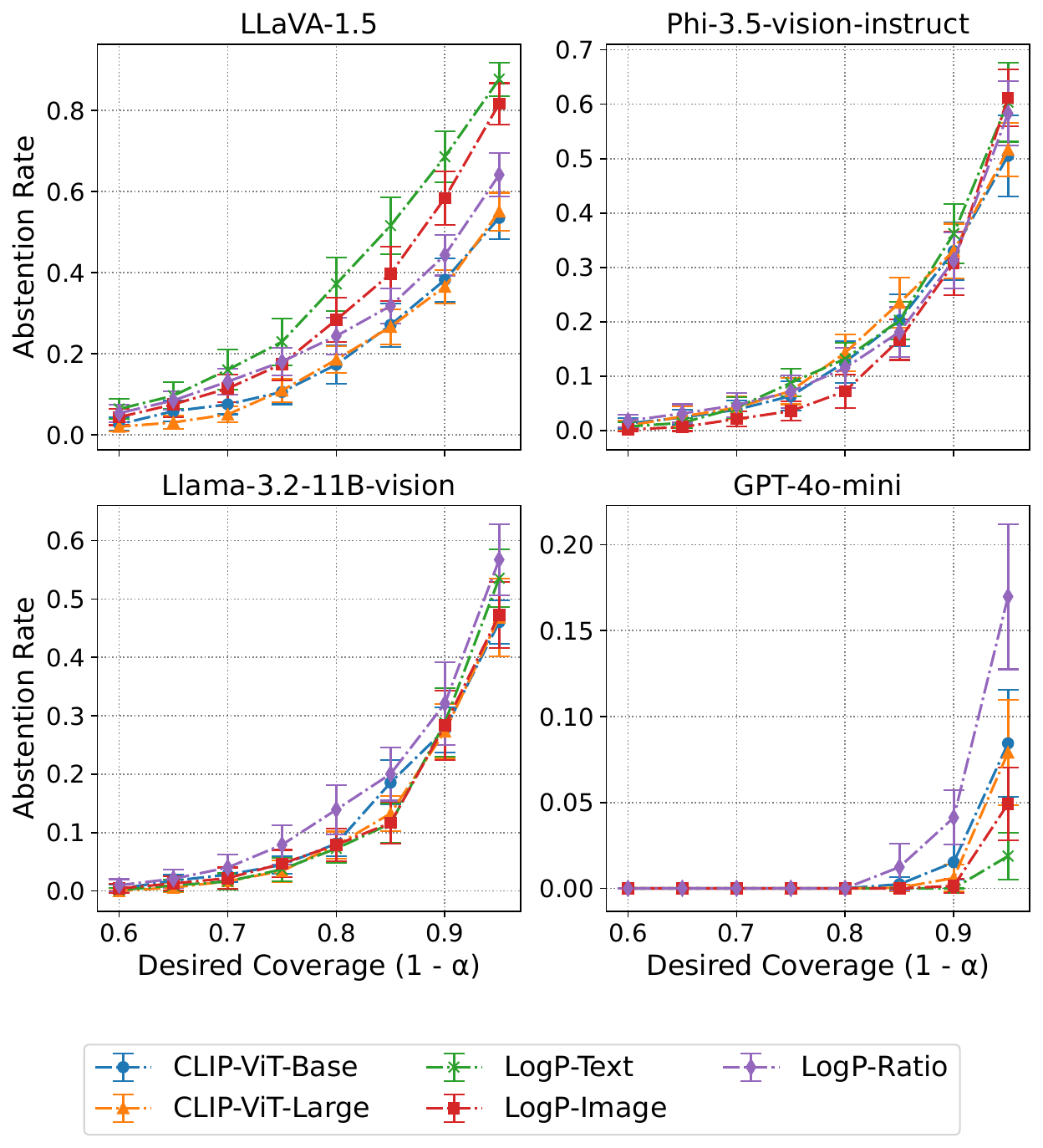}
        \caption{$\lambda=1$}
    \end{subfigure}

    \begin{subfigure}[b]{\linewidth}
        \centering
    \includegraphics[width=0.6\linewidth]{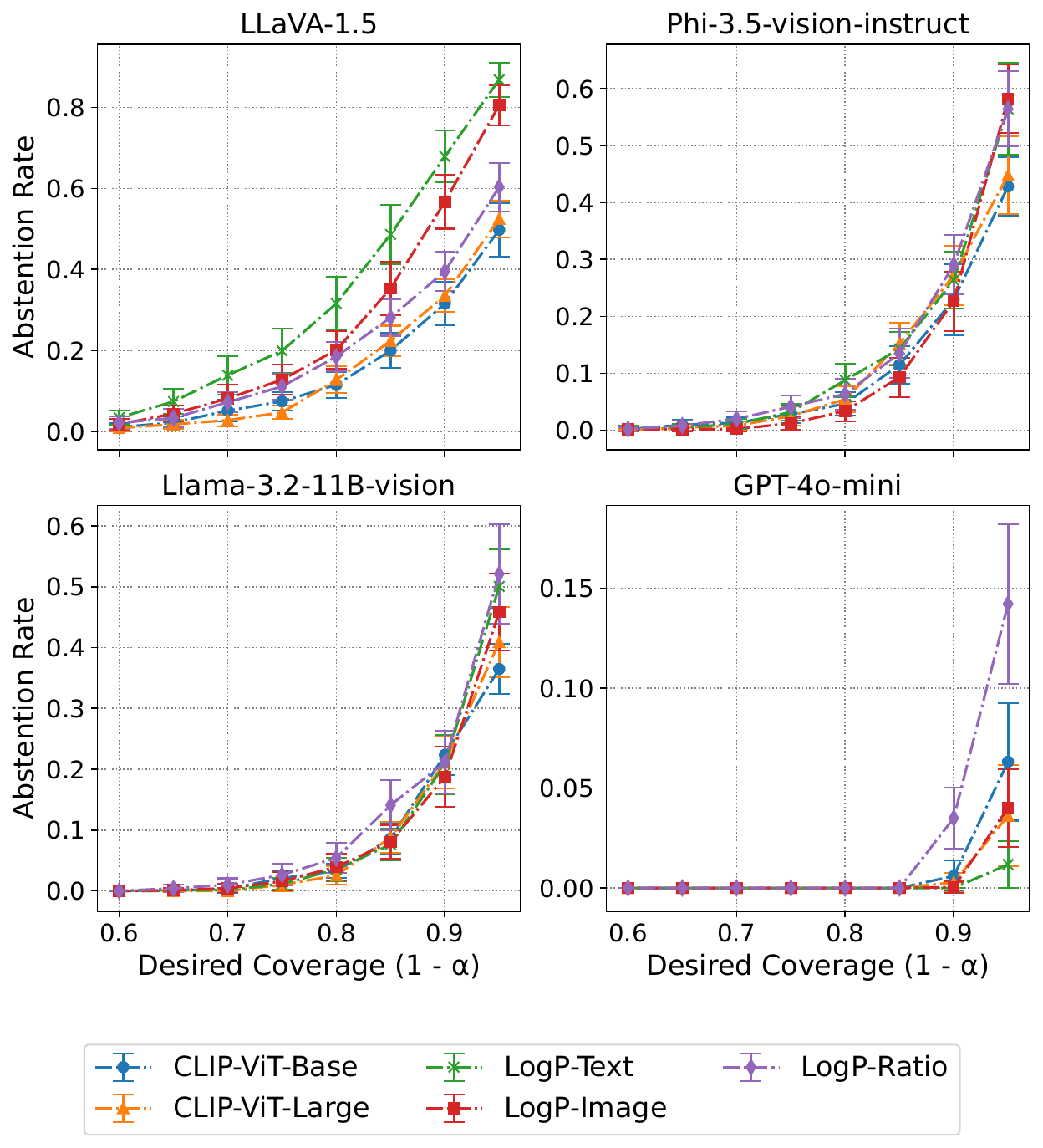}
        \caption{$\lambda=2$}
    \end{subfigure}

    \caption{Abstention rate with varying coverage using different scoring functions on the scene understanding task with different error tolerances ($\lambda$).}
    \label{fig:pope_abstent_lambda}
\end{figure}

\begin{figure}[t]
    \centering
    \begin{subfigure}[b]{0.72\linewidth}
        \centering
    \includegraphics[width=\linewidth]{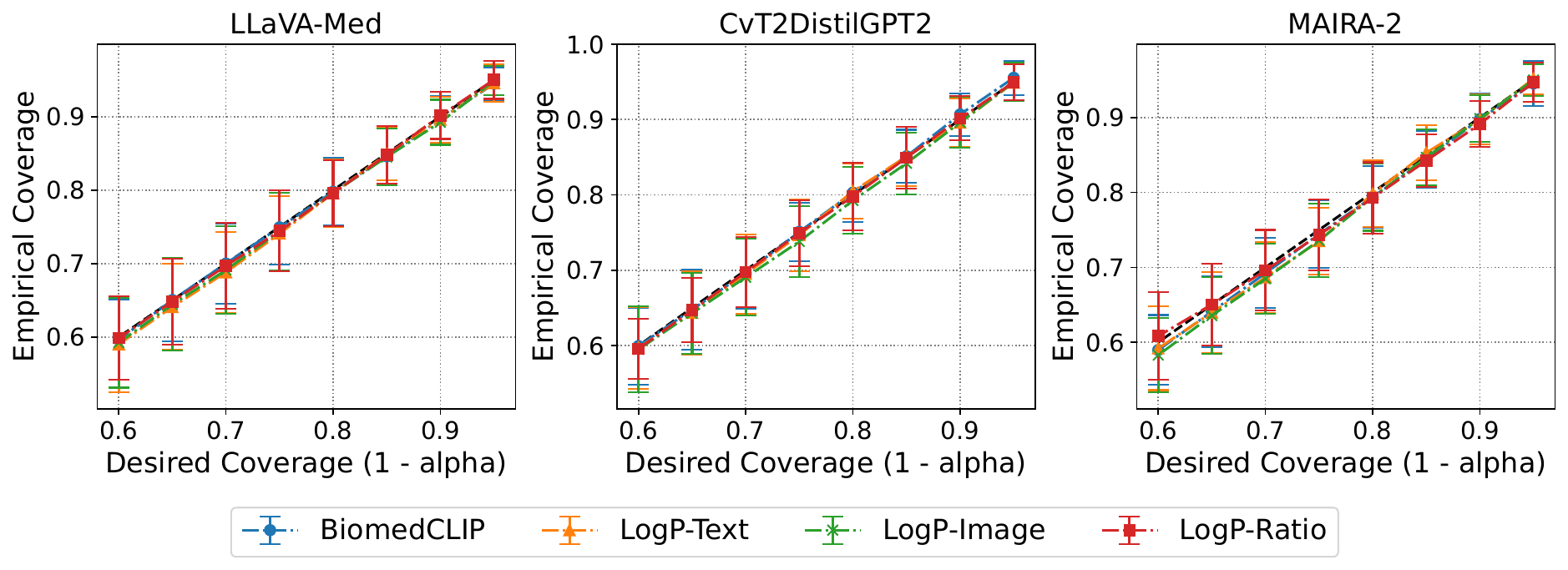}
        \caption{$\lambda=1$}
    \end{subfigure}
    
    \begin{subfigure}[b]{0.72\linewidth}
        \centering
    \includegraphics[width=\linewidth]{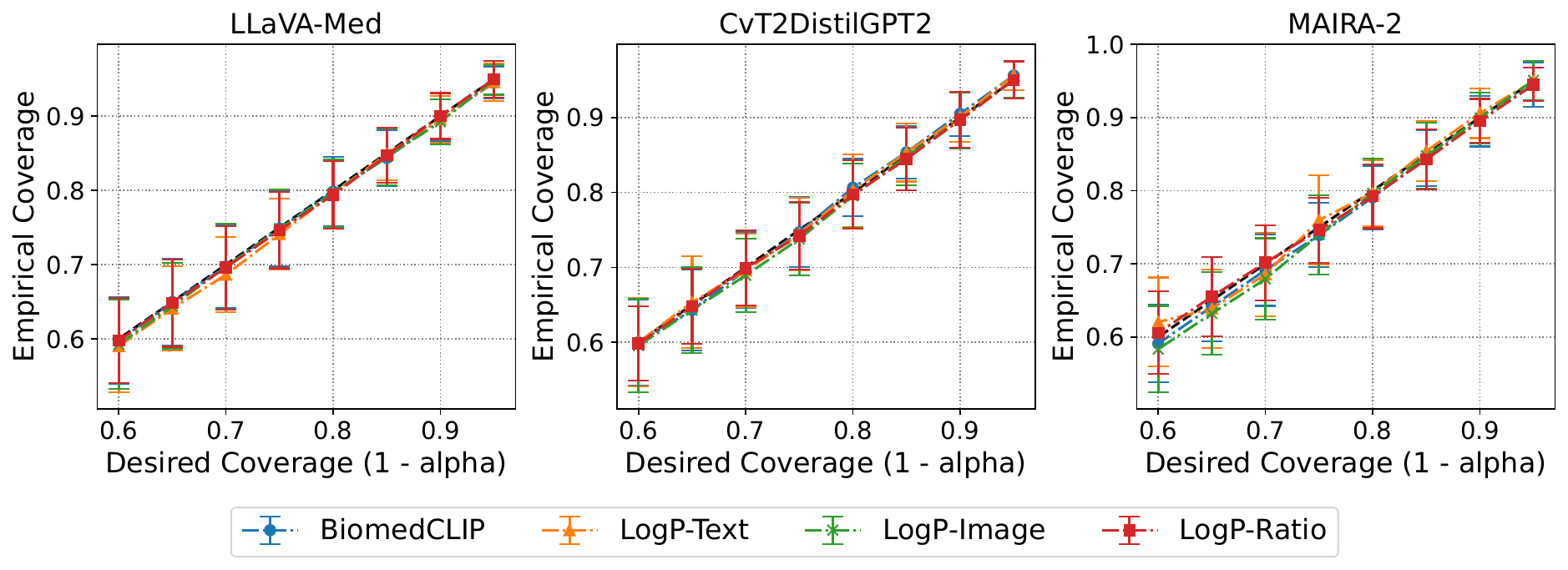}
        \caption{$\lambda=2$}
    \end{subfigure}

    \caption{Comparison of empirical and desired (theoretical) coverage on the medical report generation task with different error tolerances ($\lambda$).}
    \label{fig:mimic_coverage_lambda}
\end{figure}

\begin{figure}[t]
    \centering
    \begin{subfigure}[b]{0.72\linewidth}
        \centering
    \includegraphics[width=\linewidth]{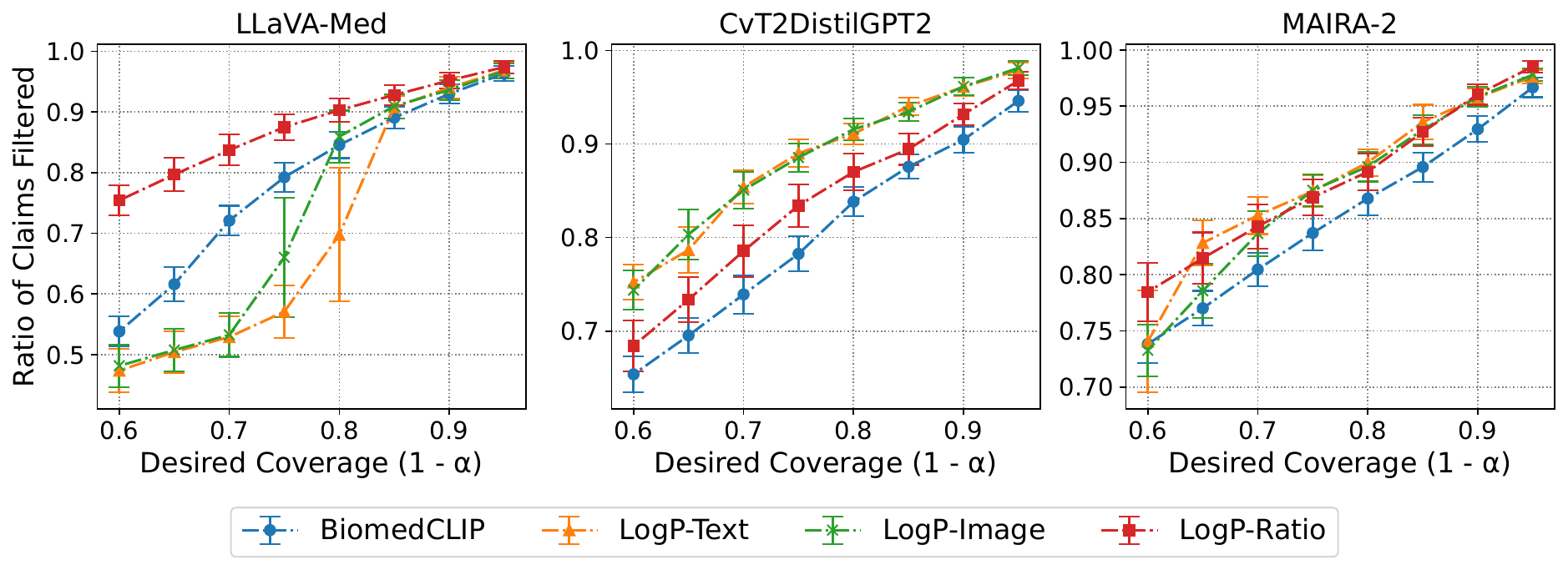}
        \caption{$\lambda=1$}
    \end{subfigure}

    \begin{subfigure}[b]{0.72\linewidth}
        \centering
    \includegraphics[width=\linewidth]{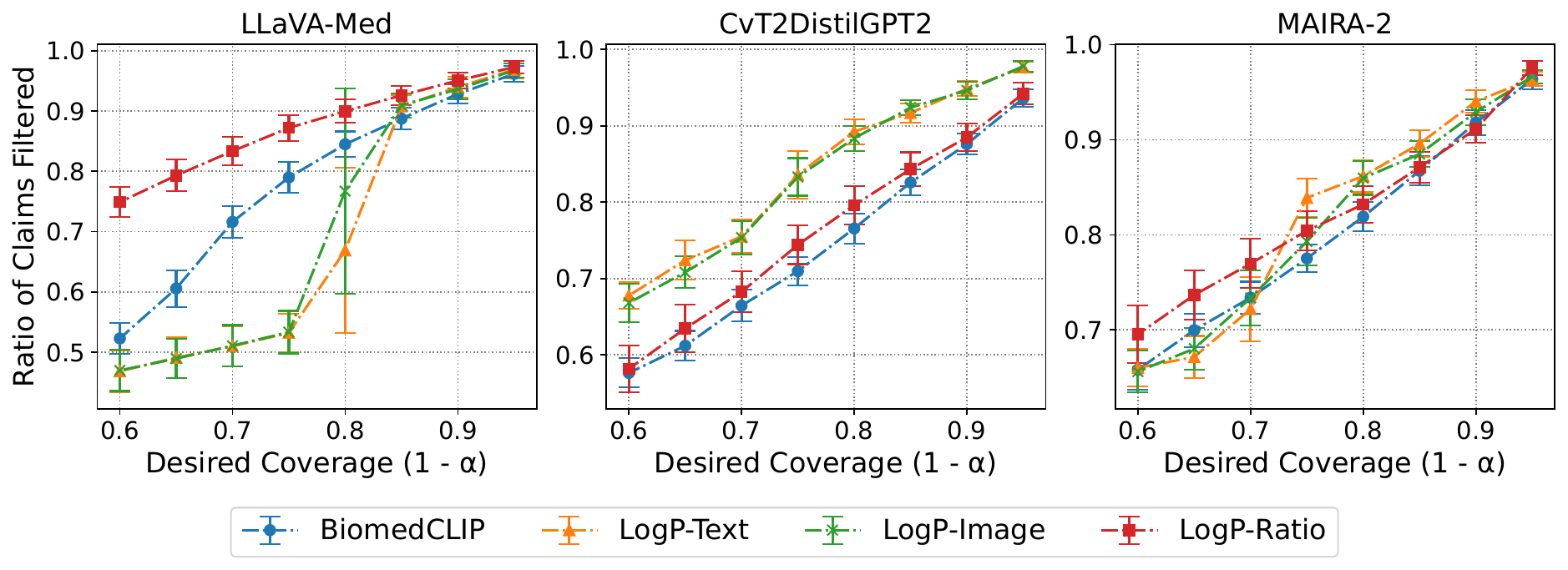}
        \caption{$\lambda=2$}
    \end{subfigure}

    \caption{Average ratio of claims filtered with varying coverage using different scoring functions on the medical report generation task with different error tolerances ($\lambda$).}
    \label{fig:mimic_cont_lambda}
\end{figure}

\begin{figure}[t]
    \centering
    \begin{subfigure}[b]{0.72\linewidth}
        \centering
    \includegraphics[width=\linewidth]{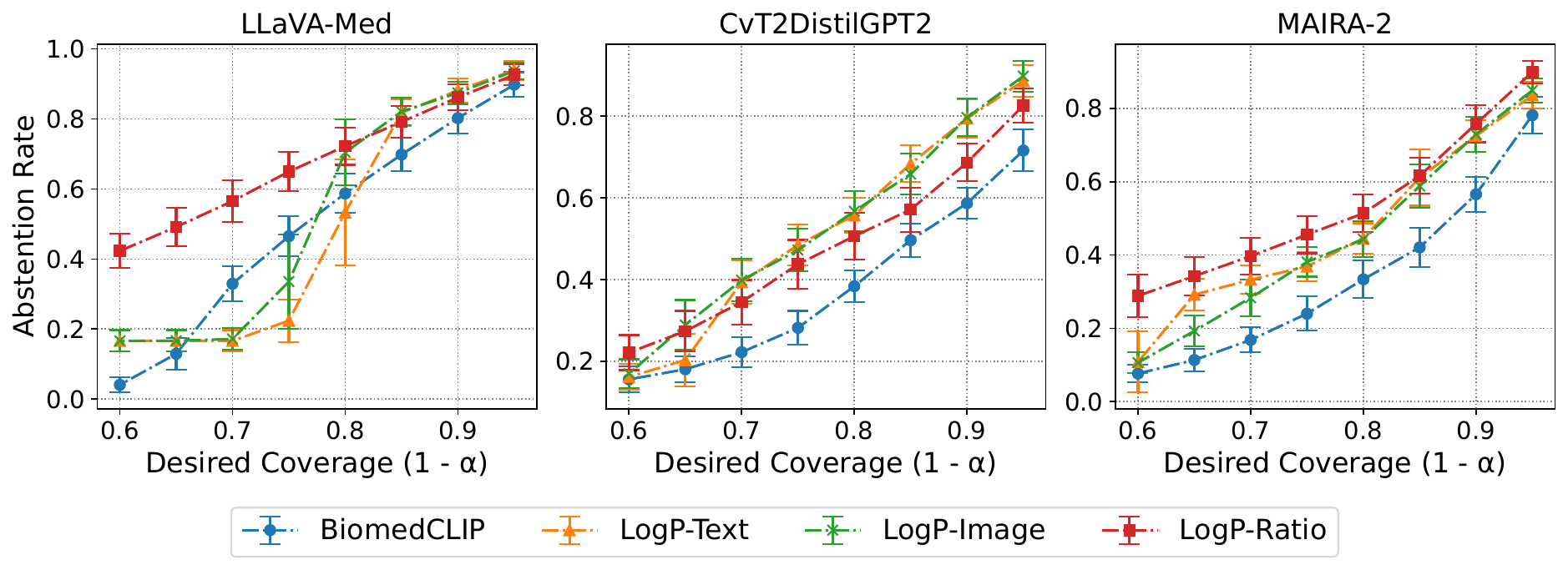}
        \caption{$\lambda=1$}
    \end{subfigure}

    \begin{subfigure}[b]{0.72\linewidth}
        \centering
    \includegraphics[width=\linewidth]{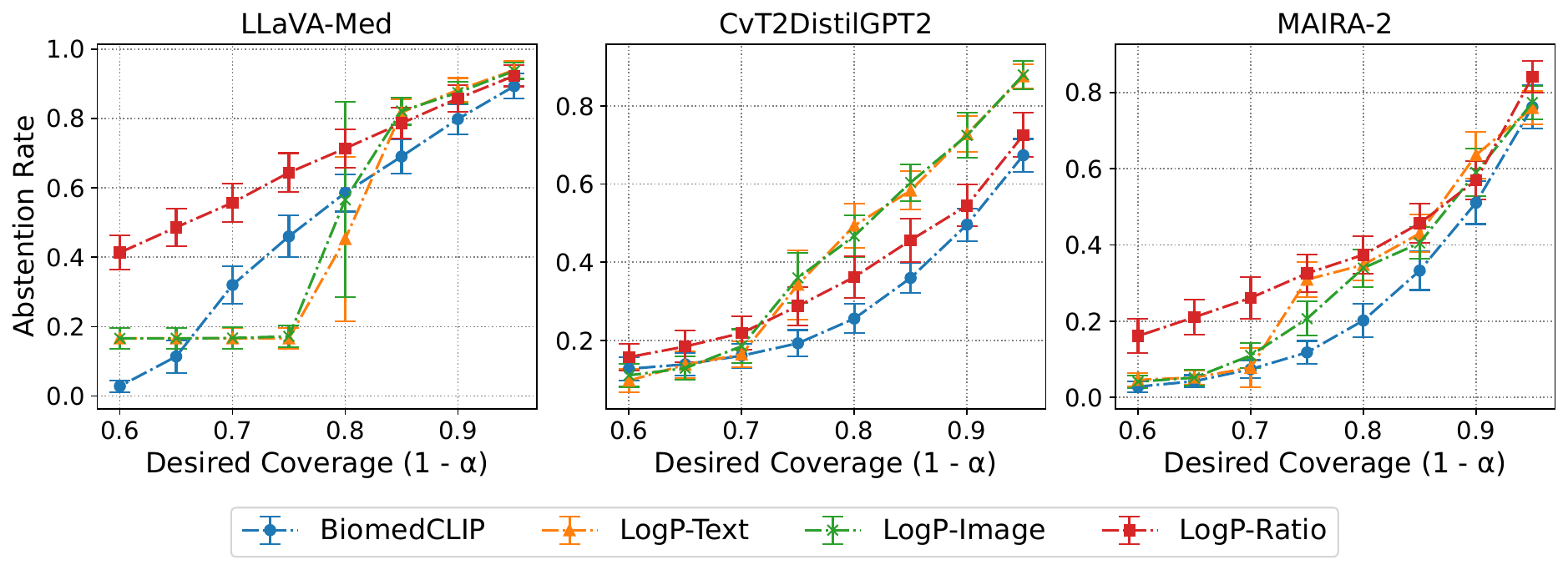}
        \caption{$\lambda=2$}
    \end{subfigure}

    \caption{Abstention rate with varying coverage using different scoring functions on the medical report generation task with different error tolerances ($\lambda$).}
    \label{fig:mimic_abstent_lambda}
\end{figure}

\begin{figure}[t]
    \centering
    \begin{subfigure}[b]{0.7\linewidth}
        \centering
    \includegraphics[width=\linewidth]{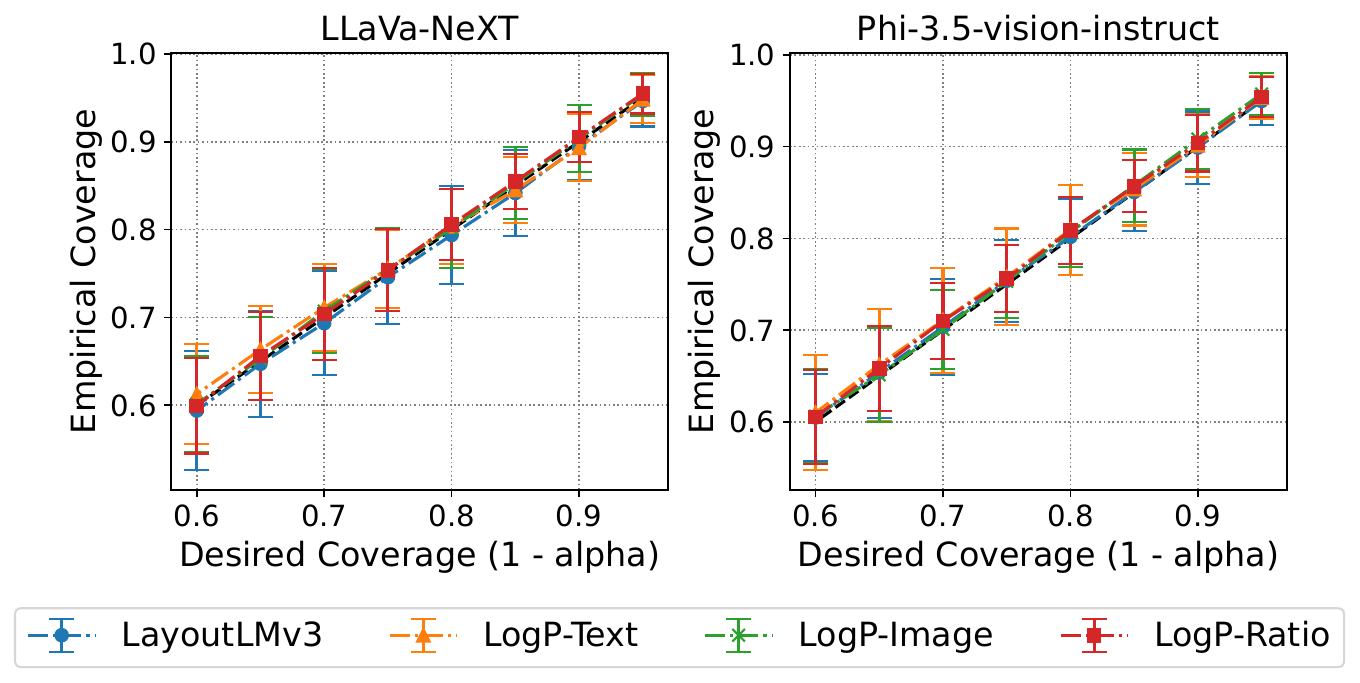}
        \caption{$\lambda=1$}
    \end{subfigure}
    
    \begin{subfigure}[b]{0.7\linewidth}
        \centering
    \includegraphics[width=\linewidth]{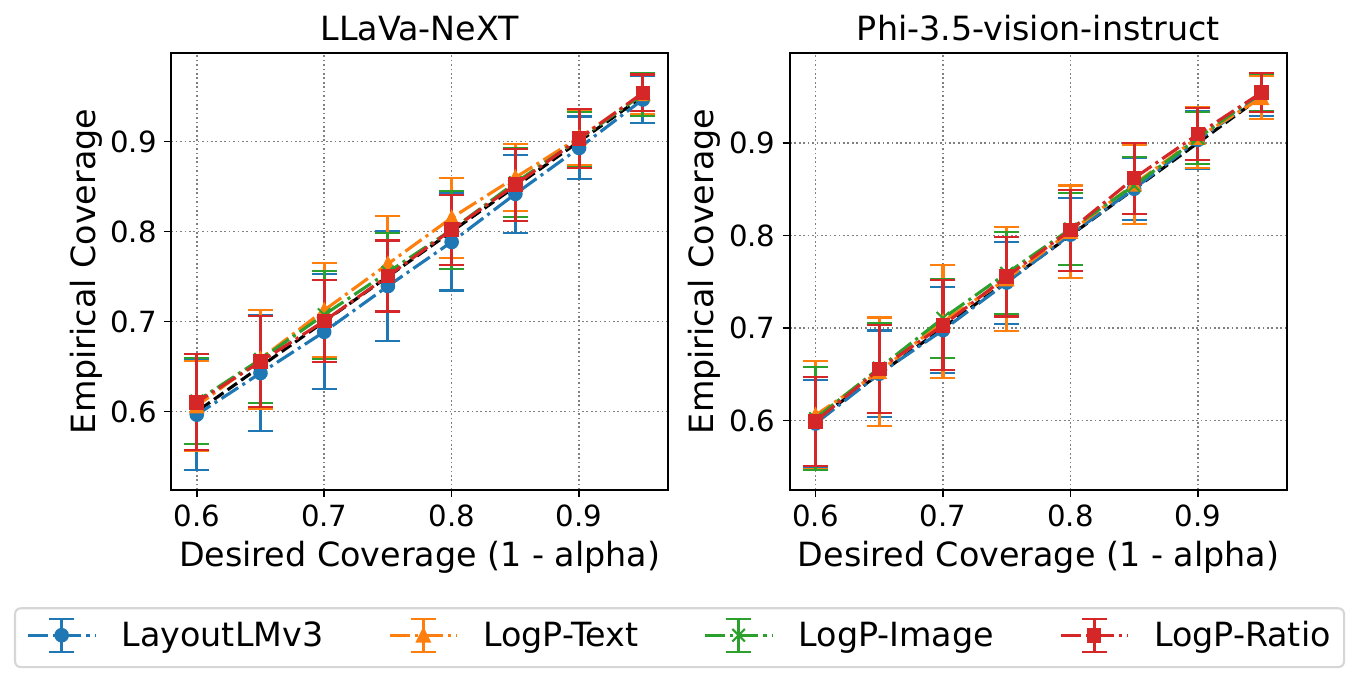}
        \caption{$\lambda=2$}
    \end{subfigure}

    \caption{Comparison of empirical and desired (theoretical) coverage on the document understanding task with different error tolerances ($\lambda$).}
    \label{fig:sroie_coverage_lambda}
\end{figure}

\begin{figure}[t]
    \centering
    \begin{subfigure}[b]{0.7\linewidth}
        \centering
    \includegraphics[width=\linewidth]{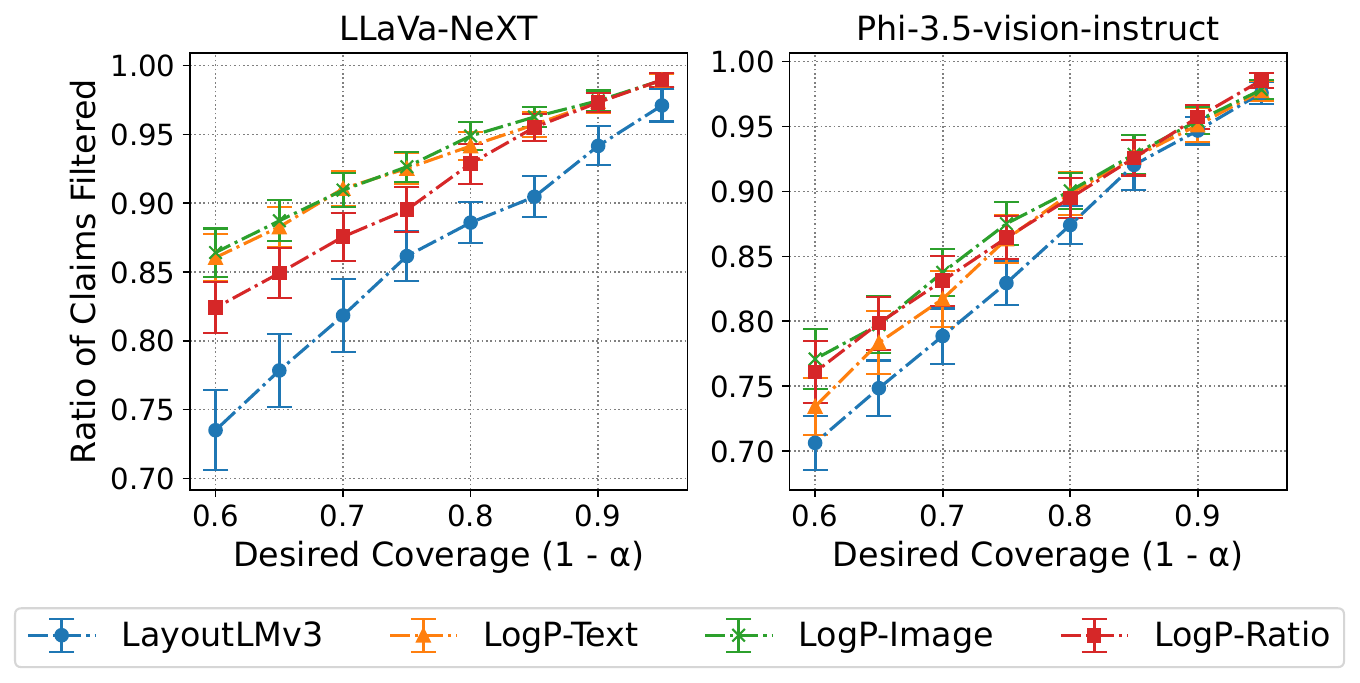}
        \caption{$\lambda=1$}
    \end{subfigure}

    \begin{subfigure}[b]{0.7\linewidth}
        \centering
    \includegraphics[width=\linewidth]{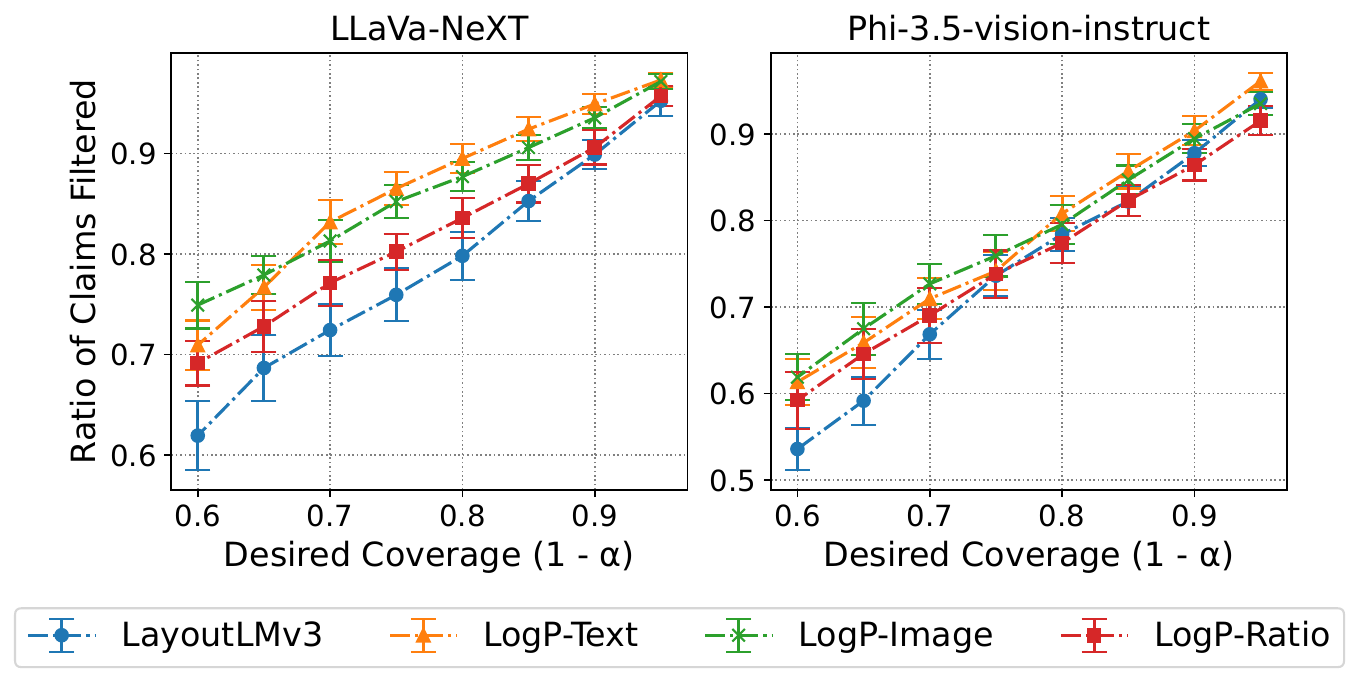}
        \caption{$\lambda=2$}
    \end{subfigure}

    \caption{Average ratio of claims filtered with varying coverage using different scoring functions on the document understanding task with different error tolerances ($\lambda$).}
    \label{fig:sroie_cont_lambda}
\end{figure}

\begin{figure}[t]
    \centering
    \begin{subfigure}[b]{0.7\linewidth}
        \centering
    \includegraphics[width=\linewidth]{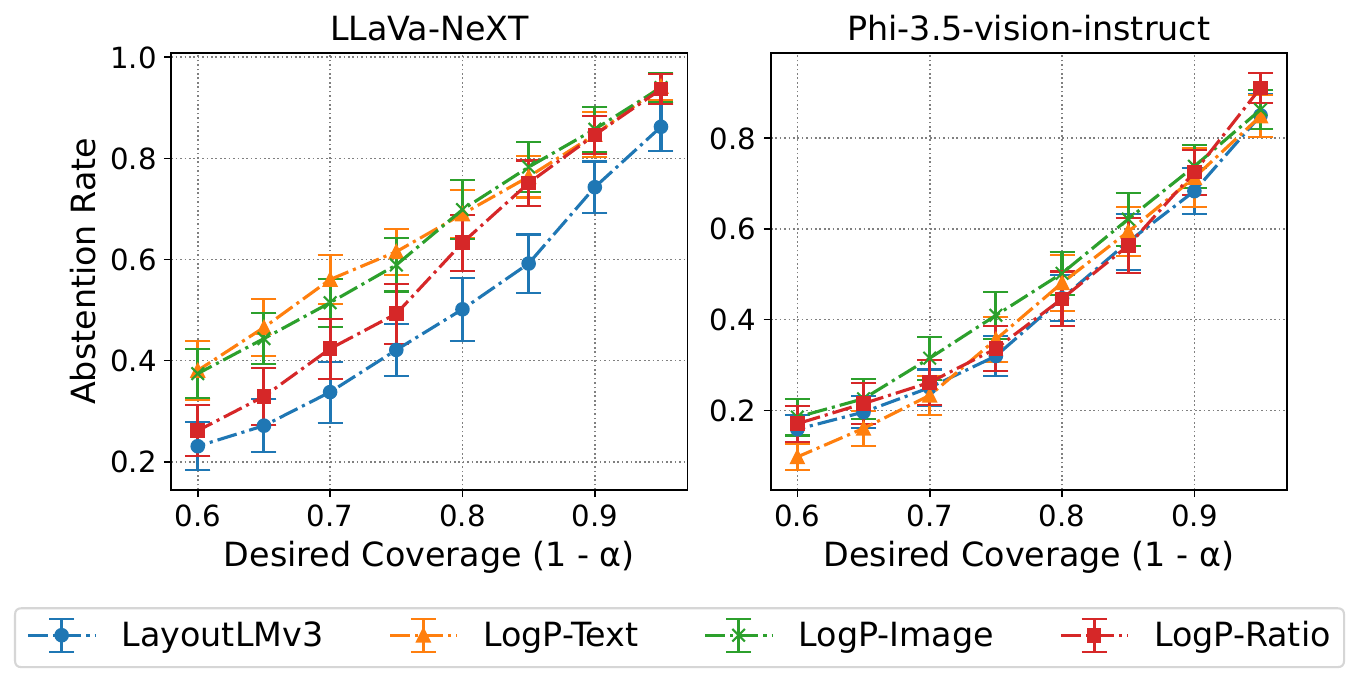}
        \caption{$\lambda=1$}
    \end{subfigure}

    \begin{subfigure}[b]{0.7\linewidth}
        \centering
    \includegraphics[width=\linewidth]{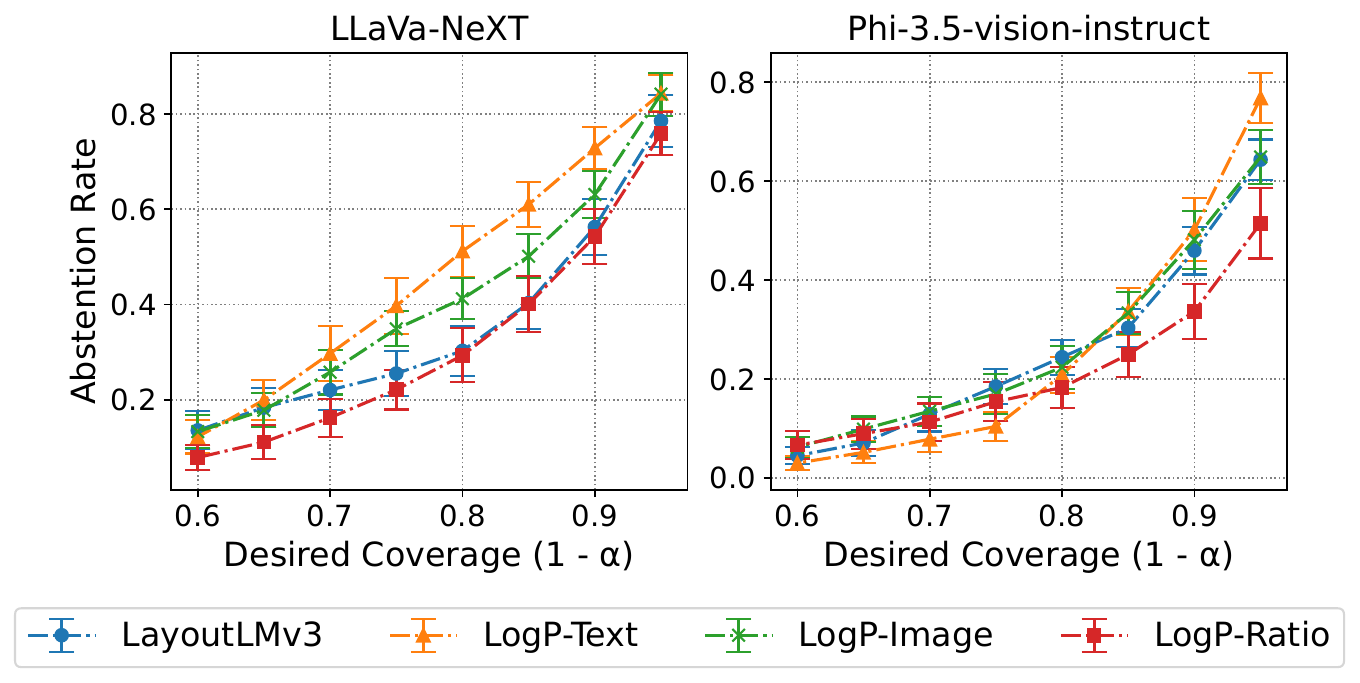}
        \caption{$\lambda=2$}
    \end{subfigure}

    \caption{Abstention rate with varying coverage using different scoring functions on the document understanding task with different error tolerances ($\lambda$).}
    \label{fig:sroie_abstent_lambda}
\end{figure}

\subsection{Examples}
We include additional examples of applying \methodName~to control error in responses from LVLMs for scene understanding (Fig.~\ref{fig:pope_example}), medical report generation (Fig.~\ref{fig:mimic_example}), and document understanding (Fig.~\ref{fig:sroie_example}).

\begin{figure*}[t]
    \centering
    \includegraphics[width=\linewidth]{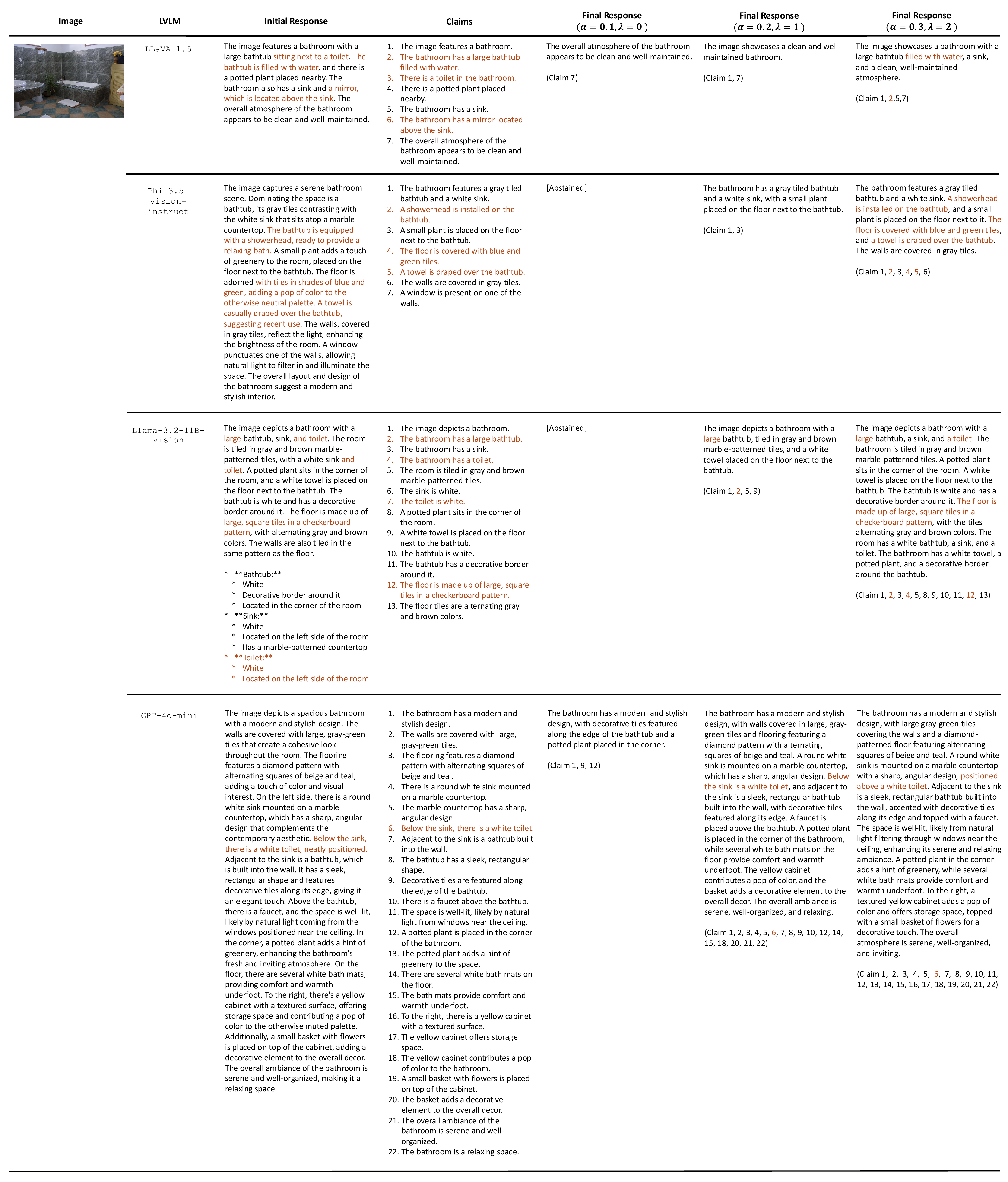}
    \caption{Examples of responses on the scene understanding task.}
    \label{fig:pope_example}
\end{figure*}

\begin{figure*}[t]
    \centering
    \includegraphics[width=\linewidth]{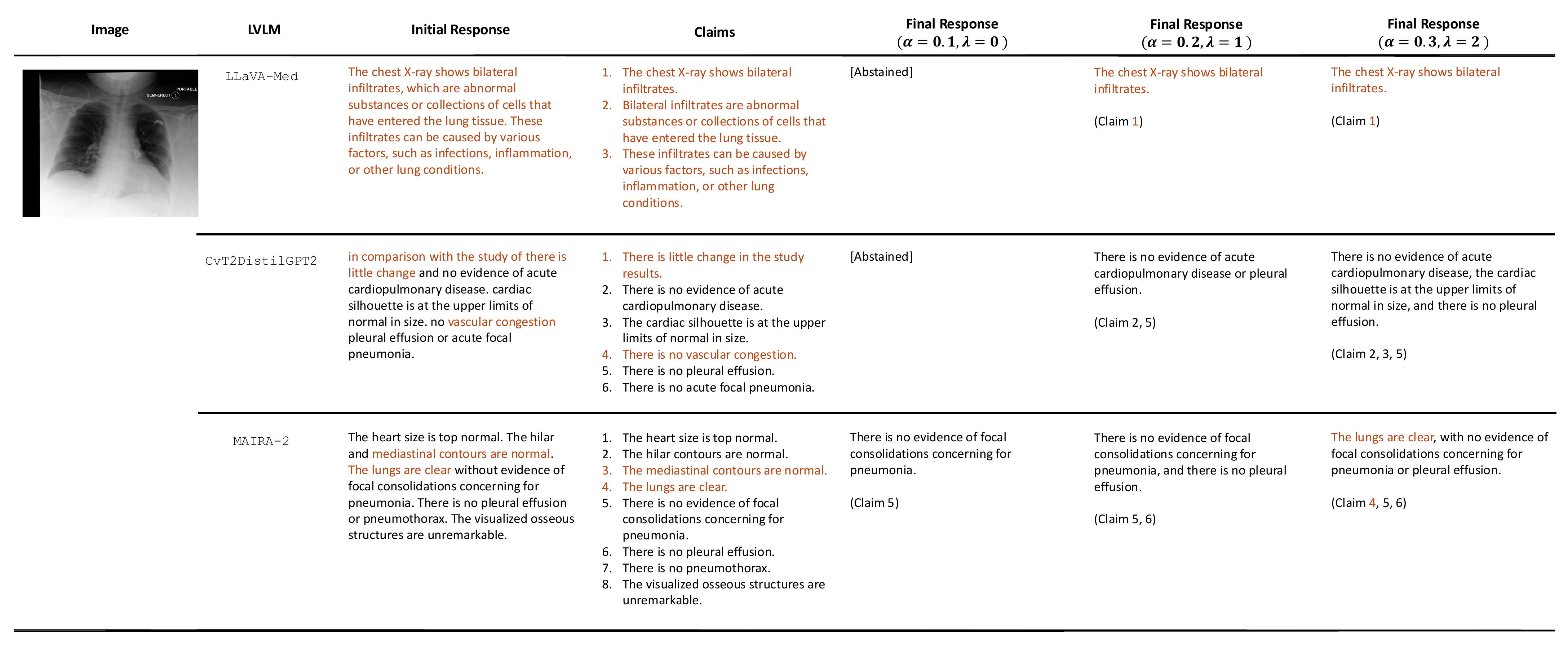}
    \caption{Examples of responses on the medical report generation task.}
    \label{fig:mimic_example}
\end{figure*}

\begin{figure*}[t]
    \centering
    \includegraphics[width=\linewidth]{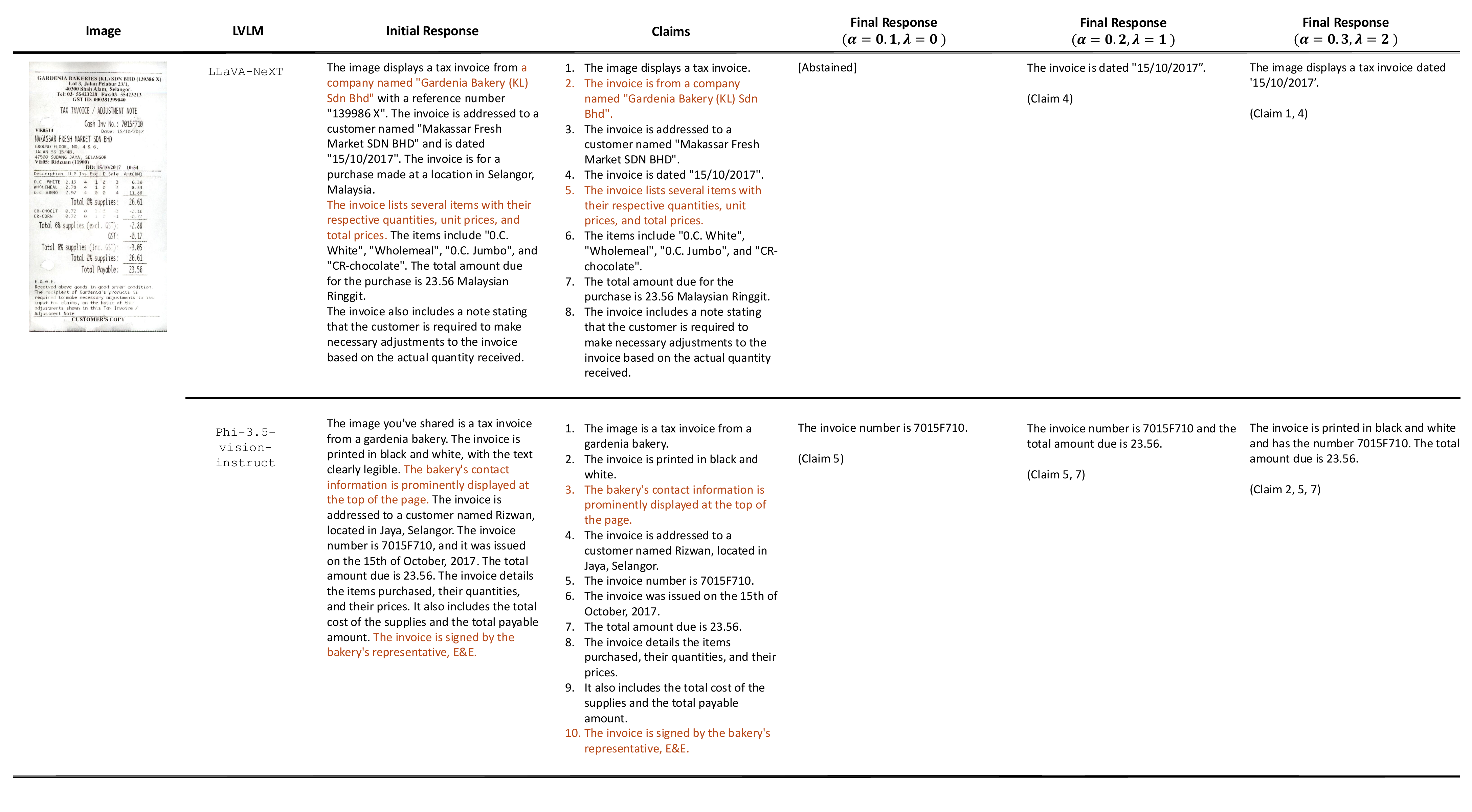}
    \caption{Examples of responses on the document understanding task.}
    \label{fig:sroie_example}
\end{figure*}

\section{Additional Discussions}

Our LVLM factuality framework follows the standard conformal prediction setting by assuming data exchangeability (a weaker notion than IID). In cases where this may not hold, there are alternative strategies that one could invoke, such as periodically updating the calibration set or applying a discount factor to older samples, as discussed in the theoretical analysis by Barber \textit{et al.}~\cite{barber2023conformal}.
For significant shifts, such as out-of-domain data, incorporating an additional OOD detection layer~\cite{li2024you} can help to empirically preserve coverage.

\end{document}